\documentclass[11pt]{amsart}
\usepackage{amsfonts}
\usepackage{amsmath,amssymb}
\usepackage[foot]{amsaddr}
\usepackage{amsthm}
\usepackage{graphicx}
\usepackage{mathtools}
\usepackage{mathrsfs}
\usepackage{xcolor}
\usepackage{textcomp}
\usepackage{booktabs}
\usepackage{array}
\usepackage{enumitem}
\usepackage{placeins}
\usepackage{microtype}
\usepackage{xurl}
\usepackage[numbers,sort&compress]{natbib}
\usepackage[colorlinks=true,linkcolor=blue,citecolor=blue,urlcolor=blue,psdextra]{hyperref}

\numberwithin{equation}{section}
\allowdisplaybreaks
\newtheorem{theorem}{Theorem}[section]
\newtheorem{proposition}[theorem]{Proposition}
\newtheorem{lemma}[theorem]{Lemma}
\newtheorem{corollary}[theorem]{Corollary}
\theoremstyle{definition}
\newtheorem{definition}[theorem]{Definition}
\theoremstyle{remark}
\newtheorem{remark}[theorem]{Remark}

\newcommand{\E}{\mathbb E}
\newcommand{\Pp}{\mathbb P}
\newcommand{\R}{\mathbb R}
\newcommand{\1}{\mathbf 1}
\newcommand{\cF}{\mathcal F}
\newcommand{\cP}{\mathcal P}
\newcommand{\relu}{\sigma}
\newcommand{\dd}{\,\mathrm d}
\DeclareMathOperator{\Var}{Var}
\DeclareMathOperator{\Cov}{Cov}

\newcommand{\tblcap}[1]{\refstepcounter{table}\par\vskip6pt\noindent\parbox{\linewidth}{\small\textbf{\tablename~\thetable.} #1\par}\par\vskip4pt}

\begin{document}
\bibliographystyle{abbrvnat}
\renewcommand{\baselinestretch}{1.05}

\title[Affine geometry of Gaussian ReLU networks]
{Affine Geometry of Gaussian ReLU Networks via Conditional Kac--Rice Formulas}

\author{Recep {\"O}zkan}
\address{Department of Mathematics, Middle East Technical University (METU), Ankara, Turkey}
\curraddr{Department of Mathematics, Aarhus University, Aarhus, Denmark}
\email{rozkan@metu.edu.tr}

\author{Christian Hirsch}
\address{Department of Mathematics, Aarhus University, Aarhus, Denmark}
\email{hirsch@math.au.dk}


\begin{abstract}
We study how the affine geometry of finite ReLU networks is created at random initialization and reorganized by supervised training. We call a sign-changing zero of a hidden preactivation an activation switch and a point where the scalar network output is nondifferentiable a scalar kink. For one-dimensional input, conditioning on the preceding layers makes each preactivation Gaussian and affine on the cells of a random finite partition. This yields an exact finite-width conditional Kac--Rice formula for the expected number of activation switches along an input interval. For fixed depth and proportionally growing widths, the resulting switch intensities converge to explicit deterministic limits. A visibility estimate shows that the expected number of switches that do not produce scalar kinks is negligible, yielding an explicit leading formula for the expected number of scalar kinks and hence affine regions. In higher input dimensions $d\ge2$, the analogous conditional surface formula yields the leading expected $(d-1)$-dimensional Hausdorff measure of the scalar kink set. On the Breast Cancer Wisconsin data, the initialization formula accurately predicts switch counts along held-out segments. After training, switch counts decrease along within-class segments and increase along between-class segments. Thus, training redistributes rather than merely contracts affine complexity.
\end{abstract}

\maketitle

\section{Introduction}

A ReLU network represents a continuous piecewise-affine map, and its activation boundaries determine where the represented affine function changes. We study this geometry for finite Gaussian ReLU networks, both at random initialization and after supervised training. Throughout, an \emph{activation switch} is a sign-changing zero of a hidden preactivation, so that the corresponding ReLU gate changes state, while a \emph{scalar kink} is a point where the scalar network output is not differentiable.

Our first results quantify this geometry at random initialization. For one-dimensional input, we obtain explicit asymptotic formulas for the expected number of scalar kinks and hence for the expected number of maximal affine regions. In higher input dimensions $d\ge2$, we obtain the corresponding asymptotic formula for the expected $(d-1)$-dimensional Hausdorff measure of the scalar kink set. We then use the same switch statistic to study trained networks empirically. On the Breast Cancer Wisconsin data, training decreases switch counts along within-class segments and increases them along between-class segments. Thus, the statistic reveals a redistribution of affine complexity rather than a uniform contraction or expansion.

The affine geometry of ReLU networks has been studied from several perspectives. Classical work investigates how depth creates many linear regions \citep{PascanuMontufarBengio2014,MontufarPascanuChoBengio2014}, how activation changes accumulate along trajectories \citep{RaghuEtAl2017}, and how regions can be counted or related to robustness \citep{SerraTjandraatmadjaRamalingam2018,CroceAndriushchenkoHein2019}. Other work emphasizes that different regions can have different geometric roles \citep{TakaiSannaiCordonnier2021,GambaEtAl2022}. A complementary geometric literature describes the induced input-space partition itself. \citet{BalestrieroCosentinoAazhangBaraniuk2019} represent the successive partitions of piecewise-affine networks through power-diagram subdivisions, while \citet{GrigsbyLindsey2022} develop a generic transversality framework for ReLU bent-hyperplane arrangements. \citet{Masden2025} gives an algorithmic description of the face poset of the resulting polyhedral complex, and \citet{HumayunEtAl2023SplineCam} compute and visualize this geometry and the associated decision boundary exactly over prescribed input regions. \citet{HuchetteMunozSerraTsay2026} provide a recent survey of the broader polyhedral perspective. Region counts have also been developed as explicit complexity measures for related ReLU architectures: \citet{LiangXu2021BiasedReLU} analyze how additional biases change the number of linear regions, while \citet{ChenWangXiong2023LinearRegions} derive region bounds for graph convolutional networks.

For random networks, Hanin and Rolnick show that along a fixed one-dimensional path the expected number of regions grows only linearly with the number of neurons and study activation-pattern geometry before and after training \citep{HaninRolnick2019LinearRegions,HaninRolnick2019ActivationPatterns}. \citet{GoujonEtemadiUnser2024} develop a complementary stochastic framework for continuous piecewise-linear networks and bound the expected density of linear regions along one-dimensional paths in terms of depth, width, and activation complexity. The closest explicit asymptotic calculation to our one-dimensional formula is due to \citet{KoganJananthanKepner2025}. Under He scaling, they start from the infinite-width Gaussian process and obtain a Cauchy zero-crossing density. Recent work of \citet{DInvernoEtAl2026FiniteSize} shows more broadly that finite-width random networks can exhibit structure that is absent from their infinite-width mean-field description. Recent probabilistic work has developed complementary asymptotic descriptions
of Gaussian neural models, including large- and moderate-deviation principles
for deep Gaussian networks \citep{Vogel2026,MacciPacchiarottiTorrisi2026}
and spectral asymptotics for Gaussian random-feature models
\citep{PaquetteXiaoZhu2026}. In a different direction, \citet{XuZhang2024Convergence} encode the piecewise-linear structure through activation domains and activation matrices to analyze convergence with increasing depth. 

Our contribution differs in that the calculation starts from the finite network. Conditional on the preceding layers, a new preactivation is Gaussian and affine on each cell of the inherited activation partition. A cellwise Kac--Rice calculation therefore gives an exact finite-width conditional formula for its expected number of activation switches. Only afterward do we pass to proportionally growing widths, where propagation of the conditional covariance data yields explicit deterministic layerwise intensities. This order of argument is important: limiting Gaussian descriptions can suggest candidate formulas, but they do not by themselves control the finite-network zero geometry.

A second issue is that activation switches and scalar kinks are not identical. A switch need not create a kink in the scalar output if every downstream route from the switching neuron is blocked, while different active routes could in principle cancel. Gaussian genericity excludes nontrivial cancellation, and the probability that an entire downstream layer is inactive is exponentially small in its width. A visibility estimate therefore transfers the expected switch count to the expected scalar kink count. In higher dimensions, the same finite-network mechanism is combined with a conditional surface formula to obtain the leading Hausdorff measure of the scalar kink set. Finite-width coarea arguments for random ReLU boundaries were developed by \citet{HaninRolnick2019LinearRegions}, while \citet{DiLilloMarinucciSalviVigogna2025Fractal} study boundary volumes of Gaussian fields after the infinite-width limit. Crofton's formula gives a direct bridge between the higher-dimensional surface measure and one-dimensional zero counts \citep{SchneiderWeil2008}.

The geometry induced by trained networks has likewise been studied from several viewpoints. \citet{SzymanskiMcCaneAtkinson2022Conceptual} quantify local network geometry through tangent spaces and relate their complexity to class separation. \citet{JiaEtAl2022ActivationSpaces} analyze the evolution of class geometry through neural activation spaces using convex-hull approximations, while \citet{BenfenatiMarta2023Riemannian} reconstruct input-space equivalence classes from a differential-geometric description of trained networks. More directly related to boundary redistribution, \citet{HumayunBalestrieroBaraniuk2024} use the density of linear regions near data as a local-complexity statistic and report that, late in training, regions migrate away from data points toward the decision boundary. \citet{PatelMontufar2025} formalize local complexity as the density of linear regions over an input distribution and connect its reduction to learned low-dimensional representations and optimization. At the level of individual neurons, \citet{ChenGe2024} analyze the evolution of ReLU activation boundaries under stochastic optimization.

These works establish that training can reorganize the spatial distribution of ReLU boundaries. Our empirical study asks a different question, using a chord-wise switch count whose behavior at random initialization is supplied by the finite-width theory above. The initialization formula accurately predicts the measured switch counts along held-out data segments, after which the comparison of within-class, between-class, and off-manifold segments reveals strongly class-dependent changes. Thus, the contribution of the trained-data experiment is not a claim that boundary redistribution itself is new, but a class-conditioned measurement of that redistribution with an analytic Gaussian initialization benchmark. We make no Gaussian assumption for trained parameters; the post-training conclusions are empirical.

The remainder of the paper develops these three parts in turn. Section~\ref{sec:main-asymptotic-theorem} states the  scalar-kink asymptotics. The finite-network Kac--Rice calculation, covariance propagation, and visibility mechanism are then made explicit, followed by finite-width numerical verification and the trained-data study. Complete proofs and the broader numerical checks are collected in the appendices.

\section{Model and main results}
\label{sec:main-asymptotic-theorem}
Let the input be $x\in\R^d$, the scalar output be $f_N(x)$, and the hidden depth $L$ be fixed. Hidden layer $\ell$ has width $n_\ell$, with $N=\sum_{\ell=1}^L n_\ell$ and $n_\ell/N\to\alpha_\ell\in(0,1)$. The activation is $\relu(u)=u_+$. The first and later layers are
\[
\begin{aligned}
 z_{1,j}(x)&=b_{1,j}+\langle W_{1,j},x\rangle,\\
 z_{\ell,j}(x)&=b_{\ell,j}+\frac1{\sqrt{n_{\ell-1}}}
 \sum_iW_{\ell,ji}h_{\ell-1,i}(x),\qquad \ell\ge2.
\end{aligned}
\]
with $h_{\ell,j}=\relu(z_{\ell,j})$ and $f_N=b_{\rm out}+n_L^{-1/2}\sum_j a_jh_{L,j}$. All parameters are independent and centered Gaussian, with $\Var(b_{\ell,j})=\beta_\ell>0$, $W_{1,j}\sim N(0,\gamma_1I_d)$, and $\Var(W_{\ell,ji})=\gamma_\ell$ for $\ell\ge2$. Define \(A_1=\beta_1\), \(B_1=\gamma_1\), and
\[
 A_\ell=\beta_\ell+\frac{\gamma_\ell A_{\ell-1}}2,
 \quad B_\ell=\frac{\gamma_\ell B_{\ell-1}}2.
\]

\subsection{One-dimensional input}
We first specialize the common network to a one-dimensional input. The scalar kink set is then a finite boundary set, and its counting measure is the quantity that matches the higher-dimensional Hausdorff measure below. The special one-dimensional geometry also recovers the number of maximal affine intervals by adding one to this kink count. Let $I=[a,b]\subset\R$ and define
\[
\begin{aligned}
 \mathcal R_N(I)&:=\{t\in\operatorname{int}(I):
 f_N(t)\text{ is not differentiable}\},\\
 R_N(I)&:=\#\mathcal R_N(I)=\mathcal H^0(\mathcal R_N(I)),\\
 \rho_\ell(t)&:=\frac{\sqrt{A_\ell B_\ell}}{\pi(A_\ell+B_\ell t^2)}.
\end{aligned}
\]
Because $f_N$ is continuous and piecewise affine, $1+R_N(I)$ is exactly the number of maximal affine intervals. The first theorem identifies the leading expected kink count through the layerwise intensities $\rho_\ell$.

\begin{theorem}[Scalar kinks in one dimension]
\label{thm:affine-region-asymptotics}
Under the preceding fixed-depth proportional-width Gaussian model, for every compact interval $I\subset\R$,
\[
\begin{aligned}
 \E R_N(I)&=\sum_{\ell=1}^L n_\ell\int_I\rho_\ell(t)\,\dd t+o(N),\\
 \frac{\E R_N(I)}N&\longrightarrow
 \sum_{\ell=1}^L\alpha_\ell\int_I\rho_\ell(t)\,\dd t.
\end{aligned}
\]
\end{theorem}

For constant variances $\beta_\ell=\beta$ and $\gamma_\ell=\gamma$, and setting 
\(B_\ell=\gamma(\gamma/2)^{\ell-1},\) the integral is explicit:
\[
\int_a^b\rho_\ell(t)\,\dd t
=\frac1\pi\big[\arctan\big(b\sqrt{B_\ell/A_\ell}\big)-\arctan\big(a\sqrt{B_\ell/A_\ell}\big)\big].
\]
Under He scaling, $\gamma_\ell=2, \beta_\ell=\sigma_b^2$ give \( A_\ell=\ell\sigma_b^2, B_\ell=2\)
\[
 \rho_\ell(t)={\sqrt{2\ell}\,\sigma_b}/(\pi(2t^2+\ell\sigma_b^2)).
\]
Up to layer indexing, this agrees with the infinite-width zero-crossing intensity of \citet{KoganJananthanKepner2025}; our argument begins with the finite network.

\subsection{Higher-dimensional input}
\label{subsec:higher-dimensional-main-result}
We next keep the input dimension fixed and replace the 1D kink count by codimension-one surface measure. The scalar output remains the principal object; hidden activation boundaries enter only as an auxiliary construction in the proof. The limiting density is determined by the covariance of the preactivation value and gradient at a fixed input. Let $d\ge2$ be fixed, let $D\subset\R^d$ be bounded, open, and convex, and let $\mathcal H^{d-1}$ denote $(d-1)$-dimensional Hausdorff measure. Define the scalar kink set
\[
 \mathcal R_N(D):=\{x\in D:f_N(x)\text{ is not differentiable}\}.
\]
Put $s_\ell(x)=A_\ell+B_\ell\|x\|^2$ and
\[
\begin{aligned}
 \Gamma_\ell(x)&=B_\ell I_d-\frac{B_\ell^2}{s_\ell(x)}xx^{\mathsf T},\\
 \rho_{\ell,d}(x)&=\frac{\E\|\Gamma_\ell(x)^{1/2}G\|}
 {\sqrt{2\pi s_\ell(x)}},\qquad G\sim N(0,I_d).
\end{aligned}
\]
The eigenvalues of $\Gamma_\ell(x)$ are $B_\ell$ in the $d-1$ transverse directions and $A_\ell B_\ell/s_\ell(x)$ radially. In particular, the density is explicit up to a one-dimensional Gaussian norm expectation and depends on $x$ only through $\|x\|$. The second theorem identifies the leading expected measure of the scalar kink set.

\begin{theorem}[Scalar kink measure]
\label{thm:higher-dimensional-boundaries}
Under the preceding fixed-depth proportional-width Gaussian model with fixed $d\ge2$, for every bounded open convex $D\subset\R^d$,
\[
 \E\mathcal H^{d-1}(\mathcal R_N(D))
 =\sum_{\ell=1}^L n_\ell\int_D\rho_{\ell,d}(x)\,\dd x+o(N).
\]
\end{theorem}

\begin{remark}[The one-dimensional specialization]
Formally setting $d=1$ gives $\rho_{\ell,1}(t)=\rho_\ell(t)$. Since $\mathcal H^0$ is counting measure,
\[
 R_N(I)=\mathcal H^0(\mathcal R_N(I)).
\]
Thus, the one-dimensional formula is precisely the $d=1$ specialization of the formula for the scalar kink measure. The additional one-dimensional fact is that $1+R_N(I)$ is the number of maximal affine intervals of $f_N$ on $I$; boundary measure has no analogous relation to the number of full-dimensional affine regions for $d\ge2$.
\end{remark}

\section{Finite-network proof overview}
The two proofs have the same architecture: an exact conditional Kac--Rice identity, propagation of empirical covariance data, and a visibility argument for the scalar output. This section records the mathematical mechanism and the formulas that determine the limiting constants without reproducing every technical estimate. Complete proofs, with the central propositions stated before the theorem deductions and their auxiliary inputs proved afterward, appear in the appendices. In one dimension, let $N_{\ell,j}(I)$ denote the number of sign-changing zeros of $z_{\ell,j}$ in $\operatorname{int}(I)$, let $S_N(I)=\sum_{\ell,j}N_{\ell,j}(I)$, and let $V_N(I)$ denote the number of these switches that create a genuine kink of the scalar output.

\subsection{Conditional Kac--Rice in one dimension} \label{subsec:Cond Kac--Rice-1d}
The first step computes the expected switch count of one neuron conditional on the preceding layers. On each inherited cell, the current preactivation is an affine Gaussian function, so its zero count can be evaluated directly by a change of variables. The resulting density depends only on the conditional variance of the value and slope and on their covariance.
In its classical one-dimensional form, the Kac--Rice formula expresses the expected number $N_0(I)$ of zeros of a sufficiently regular real-valued random process $Z$ on an interval $I$ as
\[
 \E N_0(I)=\int_I p_{Z(t)}(0)\E[|Z'(t)|\mid Z(t)=0] \,\dd t.
\]
Standard treatments include \citet{AdlerTaylor2007} and \citet{AzaisWschebor2009}. In the present setting the preactivations are only piecewise affine, so we use this mechanism separately on each inherited affine cell rather than invoke a global smooth-process formula.
Let $\cF_{\ell-1}$ be generated by the preceding layers, and on an inherited cell write $z_{\ell,j}(t)=U+Vt$. Define
\[
\begin{aligned}
 S_{\ell,N}(t)&=\Var(z_{\ell,j}(t)\mid\cF_{\ell-1}),\\
 Q_{\ell,N}(t)&=\Var(z'_{\ell,j}(t)\mid\cF_{\ell-1}),\\
 C_{\ell,N}(t)&=\Cov(z_{\ell,j}(t),z'_{\ell,j}(t)\mid\cF_{\ell-1}).
\end{aligned}
\]
On one affine cell, a zero occurs at $t$ exactly when $(U,V)=(-tV,V)$. The change of variables $(t,v)\mapsto(-tv,v)$ has absolute Jacobian $|v|$, so
\[
 \E\#\{t:U+Vt=0\}
 =\int p_{Z(t)}(0)\E[|V|\mid Z(t)=0] \,\dd t.
\]
For a centered Gaussian pair, the conditional variance of $V$ given $Z(t)=0$ is $Q-C^2/S$. Hence, summing over the inherited cells gives the exact conditional identity
\begin{equation}
\label{eq:main-kr}
\begin{split}
 &\E[N_{\ell,j}(I)\mid\cF_{\ell-1}]\\
 &=\frac1\pi\int_I\sqrt{\frac{Q_{\ell,N}(t)}{S_{\ell,N}(t)}-
 \Big(\frac{C_{\ell,N}(t)}{S_{\ell,N}(t)}\Big)^2}\,\dd t.
\end{split}
\end{equation}
The positive bias variance prevents an identically zero affine restriction. For the first layer, the unique root of $b_{1,j}+W_{1,j}t$ is a scaled ratio of two independent Gaussians, giving the same Cauchy density directly.

For fixed $t$, empirical ReLU moments propagate through the layers. The conditional covariance data satisfy
\[
 \begin{aligned}
 S_{\ell,N}(t)&=\beta_\ell+\frac{\gamma_\ell}{n_{\ell-1}}\sum_i h_{\ell-1,i}(t)^2,\\
 Q_{\ell,N}(t)&=\frac{\gamma_\ell}{n_{\ell-1}}\sum_i h'_{\ell-1,i}(t)^2,\\
 C_{\ell,N}(t)&=\frac{\gamma_\ell}{n_{\ell-1}}\sum_i h_{\ell-1,i}(t)h'_{\ell-1,i}(t),
 \end{aligned}
\]
and the half-space identities for a centered Gaussian pair give
\[
 (S_{\ell,N}(t),Q_{\ell,N}(t),C_{\ell,N}(t))
 \longrightarrow(A_\ell+B_\ell t^2,B_\ell,B_\ell t)
\]
in probability and in $L^1$. Uniform fourth moments imply uniform integrability of the random Kac--Rice density. Therefore, the expected contribution of one layer converges to $\int_I\rho_\ell$.

\subsection{Finite geometry and visibility}
The exact zero formula must be connected to the realized piecewise-affine network. For every fixed architecture, Gaussian parameters satisfy the following properties simultaneously outside a null event: inherited breakpoints are finite; a later preactivation does not vanish at an inherited breakpoint; zeros inside inherited cells are isolated and sign-changing; distinct neurons do not switch at the same input; and an active downstream path has nonzero downstream sensitivity. The key nondegeneracy statements reduce to nonzero polynomials. On an inherited cell $C$, write $h_{\ell-1,i}(t)=p_{i,C}+q_{i,C}t$ and $z_{\ell,j}(t)=U_{j,C}+V_{j,C}t$. Then
\[
\begin{aligned}
 \Var(V_{j,C}\mid\cF_{\ell-1})
 &=\frac{\gamma_\ell}{n_{\ell-1}}\sum_i q_{i,C}^2=:Q_C,\\
 \det\Cov((U_{j,C},V_{j,C})\mid\cF_{\ell-1})
 &\ge\beta_\ell Q_C.
\end{aligned}
\]
Thus, $Q_C=0$ gives a nonzero random constant, while $Q_C>0$ gives a nondegenerate affine pair. For two neurons, a common root forces the nontrivial polynomial $U_{j,C}V_{k,C}-U_{k,C}V_{j,C}$ to vanish. For a fixed downstream gate vector $g$, let $c_e>0$ denote the deterministic width normalization attached to an edge $e$ and let $w_e$ denote its weight. The sensitivity of a switching neuron $u$ is
\[
 D_u=\sum_{\pi:u\to {\rm out}}
 \Big(\prod_{e\in\pi}c_e w_e\Big)
 \Big(\prod_{v\in\pi\setminus\{u,{\rm out}\}}g_v\Big).
\]
If $g$ contains an active path, its path monomial cannot be reproduced by another path, so $D_u$ is a nonzero polynomial. Let $\cP_N(I)$ be the partition of $I$ into maximal intervals on which the complete activation pattern is constant. The genericity facts then give the exact identities $\#\cP_N(I)=1+S_N(I)$ and $R_N(I)=V_N(I)$.

\begin{figure*}[t]
\centering
\includegraphics[width=0.72\textwidth]{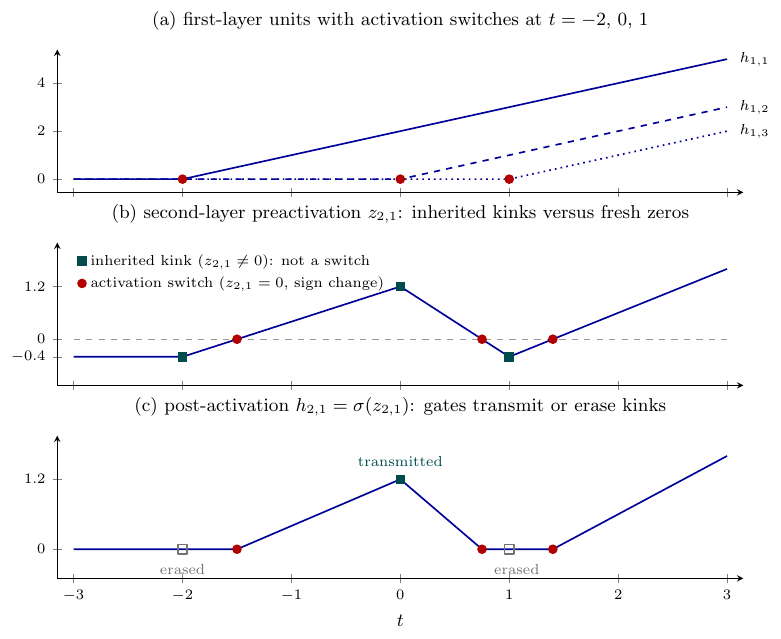}
\caption{A deterministic two-layer example separating inherited kinks from fresh activation switches. Panel (a) shows three first-layer units; in panel (b) the later preactivation inherits slope changes at earlier switches (squares) but creates new switches only at its own sign-changing zeros (circles). In panel (c), ReLU gates can transmit or erase inherited kinks. The finite Gaussian genericity argument ensures that, almost surely, these cases do not become ambiguous through coincident switches or exact downstream cancellation.}
\label{fig:switch-vs-kink}
\end{figure*}

The same genericity argument also controls visibility at the scalar output. The visibility step transfers the switch count to the scalar kink count. Finite Gaussian genericity makes distinct switch locations almost surely disjoint and reduces the local question to the downstream sensitivity of the switching neuron. On a fixed gate pattern this sensitivity is a polynomial in the remaining Gaussian edge weights. At a switch of neuron $u$, freezing downstream gates writes the derivative jump of the scalar output as the derivative jump of $h_u$ multiplied by this sensitivity. If there is an active path from $u$ to the output, the corresponding path monomial has nonzero coefficient, so the polynomial is nontrivial and does not vanish almost surely.

A switch can therefore be invisible only if some downstream hidden layer is entirely inactive. Centered Gaussian preactivations make all signs in a new layer conditionally independent and fair, giving probability $2^{-n_r}$ for this event in layer $r$. Consequently,
\begin{equation}
\label{eq:main-vis}
 0\le \E[S_N(I)-V_N(I)]
 \le\sum_{r=1}^L2^{-n_r}\E S_N(I)=o(1).
\end{equation}
The three ingredients of Theorem~\ref{thm:affine-region-asymptotics} are now in
place. The identity~\eqref{eq:main-kr} gives
$\E S_N(I)=\sum_\ell n_\ell\int_I\rho_\ell+o(N)$. The
bound~\eqref{eq:main-vis} makes the difference between $S_N(I)$ and $V_N(I)$
negligible. Finally, $R_N(I)=V_N(I)$. The complete deduction appears in
Appendix~\ref{sec:proof-affine-region-asymptotics}, and the number of maximal
affine intervals is $1+R_N(I)$.

\subsection{Fixed dimension}
The higher-dimensional argument follows the same three stages, with zero counts replaced by surface measure. The hidden activation boundaries provide an auxiliary quantity that can be computed by a conditional surface formula and then transferred to the scalar kink set by visibility. The genericity step ensures that intersections and lower-dimensional faces do not contribute to $\mathcal H^{d-1}$. For Theorem~\ref{thm:higher-dimensional-boundaries}, write $\mathcal Z_{\ell,j}(D)=\{x\in D:z_{\ell,j}(x)=0\}$ and introduce the auxiliary hidden boundary measure
\begin{equation}
\label{eq:snd}
 \mathsf S_{N,d}(D):=\sum_{\ell=1}^L\sum_{j=1}^{n_\ell}
 \mathcal H^{d-1}(\mathcal Z_{\ell,j}(D)).
\end{equation}
Conditioning on the preceding layers again yields a Gaussian field that is affine on each full-dimensional cell. The cellwise coarea formula replaces the absolute conditional slope in \eqref{eq:main-kr} by the norm of a conditional Gaussian gradient. If $S,R,C$ denote the conditional variance of the field, value-gradient covariance, and gradient covariance matrix, respectively, the local surface density is
\[
 \frac1{\sqrt{2\pi S(x)}}\E\|N(0,C(x)-R(x)R(x)^{\mathsf T}/S(x))\|.
\]
Propagation of the value-gradient empirical moments gives $(S,R,C)\to(s_\ell(x),B_\ell x,B_\ell I_d)$ and therefore the matrix $\Gamma_\ell(x)$ in Section~\ref{subsec:higher-dimensional-main-result}. The geometric step is again finite-dimensional. Replacing every ReLU gate by a deterministic mask $\varepsilon\in\{0,1\}^U$ makes each masked preactivation affine,
\[
 z_u^\varepsilon(x)=a_u^\varepsilon+\langle g_u^\varepsilon,x\rangle,
 \qquad
 Q_{\varepsilon,u}:=(a_u^\varepsilon)^2+\|g_u^\varepsilon\|^2.
\]
The bias of $u$ enters $a_u^\varepsilon$ with coefficient one, so $Q_{\varepsilon,u}$ is a nonzero polynomial and a candidate facet is almost surely a genuine hyperplane. Coincidence of two candidate hyperplanes similarly forces a nonzero polynomial relation between their affine coefficient vectors. Hence, distinct boundaries meet only in codimension at least two, which is invisible to $\mathcal H^{d-1}$. Polynomial genericity excludes coincident regular facets and nontrivial downstream cancellation. The same blocking argument as in one dimension yields
\[
\begin{aligned}
 \E\mathsf S_{N,d}(D)
 &=\sum_{\ell=1}^L n_\ell\int_D\rho_{\ell,d}(x)\,\dd x+o(N),\\
 0&\le\E\big[\mathsf S_{N,d}(D)
 -\mathcal H^{d-1}(\mathcal R_N(D))\big]=o(1),
\end{aligned}
\]
which gives the main-text reduction to the scalar-output theorem; the complete proof appears in Appendix~\ref{sec:higher-dimensional-boundaries}. 

\section{Finite-width verification at initialization}
The preceding theorems are asymptotic, so we first check how accurately the leading terms describe networks at moderate width. These experiments are not needed for the proofs. Their role is to verify the finite-width implementation before using the same switch statistic in the trained-data study.

\subsection{One-dimensional convergence}
We first test the one-dimensional formula at widths where the network can be simulated exactly along the full interval. The experiment separates the finite-width identity at depth one from the asymptotic predictions at depths two and three. For $\beta=\gamma=1$ and $I=[-1,1]$, the layer integrals are $1/2$, $1/3$, and approximately $0.2301$ for the first three layers, while at depth one the kink-count formula is exact at every finite width:
\[
 \E R_N(I)=n_1\int_I\rho_1(t)\,\dd t=n_1/2.
\]
Equivalently, the expected number of maximal affine intervals is $1+n_1/2$.
For depths two and three, the theorem guarantees only an $o(N)$ correction. Across the widths used in the experiment, the measured expectation stays within two affine regions of the leading prediction, and the normalized counts are already close to the limiting constants, as reported in Table~\ref{tab:mc-main} and Figure~\ref{fig:mc-convergence-main}.

\begin{figure*}[!ht]
\centering
\includegraphics[width=0.93\columnwidth]{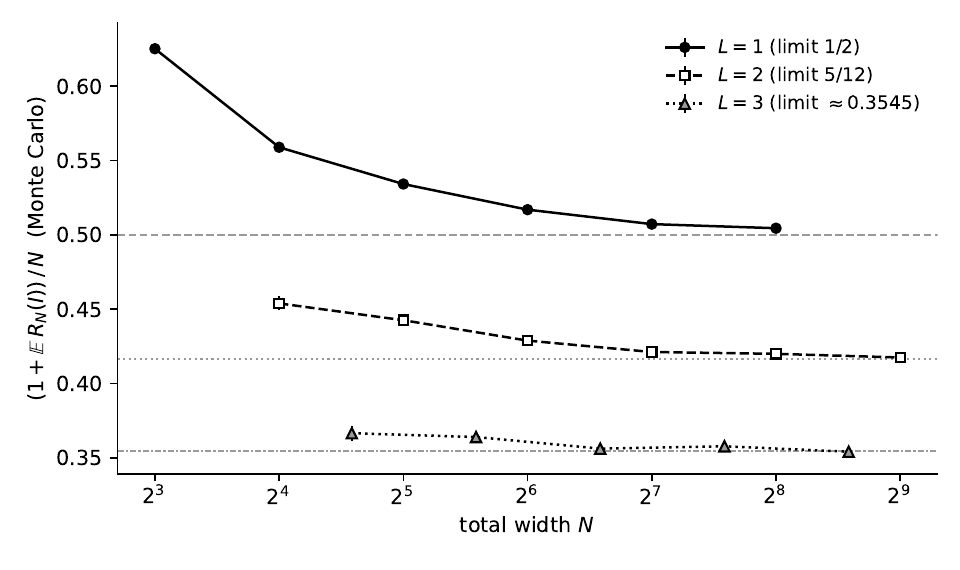}
\caption{Monte Carlo estimates of $(1+\E R_N(I))/N$, the expected number of maximal affine intervals normalized by $N$, for $\beta=\gamma=1$, $I=[-1,1]$, and equal widths. Horizontal lines are the asymptotic values $1/2$, $5/12$, and approximately $0.3545$ for depths $L=1,2,3$.}
\label{fig:mc-convergence-main}
\end{figure*}

\begin{table*}[t]
\centering
\caption{Monte Carlo estimates of $1+\E R_N([-1,1])$, the expected number of maximal affine intervals, for $\beta=\gamma=1$. The theory column is $1+\sum_\ell n_\ell\int_I\rho_\ell$; it is exact for $L=1$ and the leading prediction for $L\ge2$. These rows are a selection; Table~\ref{tab:mc} reports
the full set of runs.}
\label{tab:mc-main}
\small
\begin{tabular}{cccrr@{\qquad}cccrr}
\toprule
$L$ & widths & $N$ & MC & theory & $L$ & widths & $N$ & MC & theory\\
\midrule
1 & $16$ & 16 & $8.94\pm0.05$ & 9.00 & 2 & $16,16$ & 32 & $14.16\pm0.11$ & 14.33\\
1 & $64$ & 64 & $33.08\pm0.09$ & 33.00 & 2 & $64,64$ & 128 & $53.92\pm0.21$ & 54.33\\
1 & $256$ & 256 & $129.15\pm0.17$ & 129.00 & 3 & $64,64,64$ & 192 & $68.72\pm0.37$ & 69.06\\
2 & $128,64$ & 192 & $86.17\pm0.36$ & 86.33 & 2 & $64,128$ & 192 & $74.71\pm0.42$ & 75.67\\
\bottomrule
\end{tabular}
\end{table*}

The unequal-width rows use the same total budget $N=192$ but allocate more neurons to different layers. Their agreement with the prediction checks that the theorem is not tied to equal widths. They also illustrate the layerwise nature of the formula: for these parameters, moving width toward the earlier layer increases the expected count because the layer intensity decreases with depth.

\subsection{Higher-dimensional surface measure}
For $D=B(0,1)$, we estimate the auxiliary hidden boundary measure through Crofton's representation rather than a lattice discretization. Sampling a random direction and a random offset in the orthogonal projection of the ball gives an unbiased estimator based entirely on exact one-dimensional chord intersections. Table~\ref{tab:crofton-main} reports the ratio of the measured boundary measure to the asymptotic prediction. The ratios are close to one for dimensions $2,3,5$ and depths up to three; the full calibration table and standard errors are in Appendix~\ref{sec:higher-dimensional-numerics}.

\begin{table}[t]
\centering
\caption{Finite-network boundary measure on the unit ball with
$\beta_\ell=\gamma_\ell=1$. Predictions are the leading terms from the layerwise
surface integral; errors are one standard error over $40$ networks. These rows
are a selection covering all three input dimensions and all three depths;
Table~\ref{tab:crofton-networks} reports the full set of nine configurations.}
\label{tab:crofton-main}
\small
\begin{tabular}{ccrrr}
\toprule
$d$ & widths & prediction & estimate & ratio\\
\midrule
2 & $64$ & 76.68 & $76.22\pm1.18$ & 0.994\\
3 & $32,32$ & 107.83 & $108.30\pm1.84$ & 1.004\\
5 & $24,24,24$ & 171.86 & $167.43\pm3.35$ & 0.974\\
\bottomrule
\end{tabular}
\end{table}

These checks also illustrate why the hidden quantity remains useful as an auxiliary construction: it is directly measurable by line intersections and the visibility estimate transfers its leading term to the scalar output. In the main theorems, however, the headline quantities are the scalar kink count in one dimension and the scalar kink-set measure in fixed higher dimension; the number of one-dimensional affine regions is $1+R_N(I)$.

\section{Affine geometry before and after training}
\label{sec:wisconsin}
The initialization theory gives an analytically tractable reference geometry. We now use the same finite-network statistic to ask how supervised learning changes that geometry on real data. All post-training statements in this section are empirical; Theorems~\ref{thm:affine-region-asymptotics} and~\ref{thm:higher-dimensional-boundaries} are not asserted for trained weights.

\subsection{Design and initialization prediction}
We use the Breast Cancer Wisconsin diagnostic data, a classical benchmark for binary classification based on $30$ quantitative features extracted from digitized images of breast-mass cell nuclei \citep{StreetWolbergMangasarian1993}. The data contain $569$ observations, with $212$ malignant and $357$ benign cases, and we use $d=30$ standardized features. The principal experiments use $10$ stratified $70/30$ train/test splits with $5$
independently initialized networks per split, giving $50$ split--seed
combinations for each of the two architectures reported below. Feature
standardization is fitted on the training observations of each split and then
applied to the held-out observations. Within a split, the evaluation segments
are drawn once and held fixed across initialization, training, and all network
seeds, so every change is measured on the same segments before and after
training. The relative change is computed separately for each split--seed
combination, and the variability reported below is across those combinations.
Segments measured on one network are not treated as independent replicates, and
the splits are re-partitions of one dataset rather than independent datasets.
The wider architectures and the bias-variance sweep of
Appendix~\ref{app:wisconsin} remain single-split exploratory checks. The main network has two hidden layers of width $(64,64)$ and a scalar logit. We use He scaling in the parameterization above and hidden bias variance $\sigma_b^2=0.1$, which respects the positive-bias-variance assumption of the theory. Training uses full-batch Adam with learning rate $3\times10^{-3}$ for $300$ epochs and binary cross entropy.

For endpoints $x_0,x_1$, let $N_\ell(x_0,x_1)$ denote the total number of activation switches contributed by layer $\ell$ along the segment joining them. Write $x(t)=x_0+t(x_1-x_0)$, $t\in[0,1]$, let $v=(x_1-x_0)/\|x_1-x_0\|$, $s_0=\langle x_0,v\rangle$, $s_1=s_0+\|x_1-x_0\|$, and $y=x_0-s_0v$. The restriction to this chord is a one-dimensional Gaussian ReLU network whose limiting layer-$\ell$ intercept and slope variances are $A'_\ell=A_\ell+B_\ell\|y\|^2$ and $B_\ell$. Hence, the covariance propagation in Section~\ref{subsec:Cond Kac--Rice-1d} gives, in the fixed-depth proportional-width regime,
\begin{equation}
\label{eq:chord-prediction}
\E N_\ell(x_0,x_1)
=\frac{n_\ell}{\pi}\big[\arctan\big(s_1\sqrt{B_\ell/A'_\ell}\big)-\arctan\big(s_0\sqrt{B_\ell/A'_\ell}\big)\big]+o(n_\ell).
\end{equation}
For $\ell=1$, the remainder in~\eqref{eq:chord-prediction} is identically zero, so the displayed expression is the exact finite-width expectation. For $\ell\ge2$, however, the exact finite-width input is the conditional Kac--Rice identity~\eqref{eq:main-kr}; the deterministic arctangent expression arises only after propagation of the random conditional covariance data and is therefore a leading asymptotic prediction rather than an exact finite-width expectation. We use this leading term as the closed-form initialization prediction below, summing it over the hidden layers when predicting the total switch count.

This prediction concerns the hidden switch count $S_N$, not the scalar kink
count $R_N$ of Theorem~\ref{thm:affine-region-asymptotics}. We report $S_N$
throughout this section. The chord prediction~\eqref{eq:chord-prediction} is
layerwise, and the layer decomposition below requires the layer attribution
that $S_N$ carries and $R_N$ does not. The two agree to leading order at these
widths by the visibility bound~\eqref{eq:main-vis}, in which the expected
number of switches invisible at the output is of order $\sum_{r}2^{-n_r}$; they
agree exactly in every exploratory run of Appendix~\ref{app:wisconsin}.

No quadrature or grid discretization enters either the prediction or the measured switch count: all switches are located exactly by solving linear equations on inherited cells.

We use four segment families: within malignant, within benign, between class, and an off-manifold control whose endpoints are independent $N(0,I_{30})$ draws. The control has the same coordinatewise scale as the standardized features but is unrelated to the learned data geometry. At initialization, the closed-form prediction agrees with the measured count for
every architecture and segment family. For $(64,64)$, averaged over the $50$
split--seed combinations, the within-malignant prediction is $47.14$ against
$46.49\pm3.33$ measured, and the between-class prediction is $70.98$ against
$70.60\pm3.04$, where the stated spread is one standard deviation across
combinations. In both cases the deviation is below a quarter of that spread. The
corresponding single-split figures for all four exploratory architectures are in
Appendix~\ref{app:wisconsin}.

\subsection{Training redistributes the boundaries}
We next compare the same segment statistics before and after supervised training. The main question is whether training changes affine complexity globally or reorganizes it according to the class relation. The changes are strongly class dependent. For $(64,64)$, averaged over the $50$
split--seed combinations, total switch counts decrease by $24.1\%$ on
within-malignant and by $22.5\%$ on within-benign segments and increase by
$18.6\%$ between classes, while the off-manifold control changes by only
$-4.3\%$. The direction agrees with the mean in every one of the $50$
combinations for the three data-segment families; for the control, whose mean is
close to zero, it agrees in $49$ of $50$. The deeper $(64,64,64)$ network behaves
the same way. Table~\ref{tab:wis-main} reports both architectures with their
variability.
\begin{table*}[t]
\centering
\caption{Relative change in total switch count from initialization to epoch
$300$, over $10$ stratified splits with $5$ network seeds each. Entries are the
mean over the $50$ split--seed combinations, one standard deviation across them,
and the number of combinations in which the change has the same sign as the
mean. Test accuracy is $0.970\pm0.013$ for $(64,64)$ and $0.971\pm0.013$ for
$(64,64,64)$.}
\label{tab:wis-main}
\small
\begin{tabular}{llrrr}
\toprule
architecture & segment family & mean & sd & same sign\\
\midrule
$(64,64)$ & within malignant & $-24.1\%$ & $7.3$ & $50/50$\\
 & within benign & $-22.5\%$ & $5.7$ & $50/50$\\
 & between class & $+18.6\%$ & $5.4$ & $50/50$\\
 & off manifold (control) & $-4.3\%$ & $1.8$ & $49/50$\\
\midrule
$(64,64,64)$ & within malignant & $-27.3\%$ & $7.1$ & $50/50$\\
 & within benign & $-25.9\%$ & $5.1$ & $50/50$\\
 & between class & $+15.5\%$ & $5.3$ & $50/50$\\
 & off manifold (control) & $-9.2\%$ & $3.0$ & $50/50$\\
\bottomrule
\end{tabular}
\end{table*}

\begin{figure*}[!t]
\centering
\begin{minipage}[b]{0.48\textwidth}
  \centering
  \includegraphics[width=\textwidth]{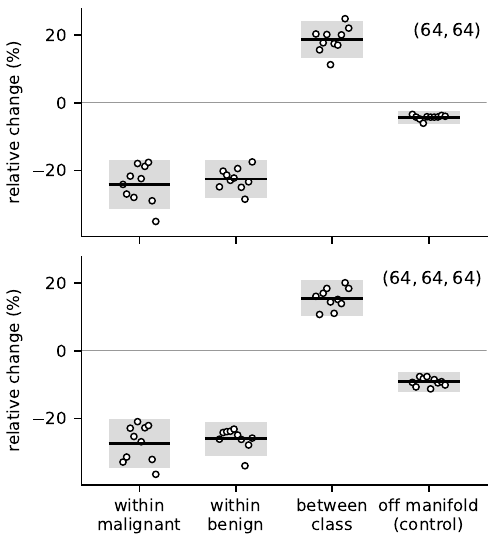}\\[2pt]
  {\small (a)}
\end{minipage}
\hfill
\begin{minipage}[b]{0.48\textwidth}
  \centering
  \includegraphics[width=\textwidth]{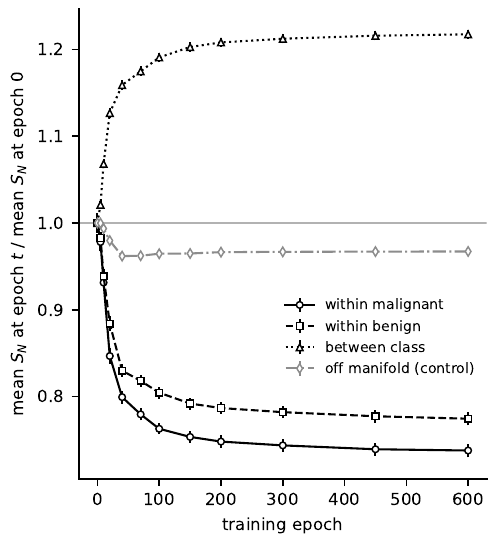}\\[2pt]
  {\small (b)}
\end{minipage}
\caption{Split-level and trajectory views of the same experiment.
(a) Results behind Table~\ref{tab:wis-main}. Open circles are the ten split
means, each averaged over the five network seeds of that split; the horizontal
bar is the mean over all $50$ split--seed combinations and the grey band is one
standard deviation across them. The within-class families lie entirely below
zero and the between-class family entirely above it in every split, while the
control stays close to zero.
(b) For the $(64,64)$ network, switch counts along within-class segments
decrease during training, between-class counts increase, and the off-manifold
control remains comparatively flat. Counts are normalized by initialization and
averaged over $20$ network seeds on the single exploratory split.}
\label{fig:wis-splits}
\end{figure*}

The separation is therefore not an artifact of one partition of the data or of
one initialization: Figure~\ref{fig:wis-splits} shows that no split reverses it
and that the between-class and within-class families do not overlap in any of
them. Relative to the off-manifold control, within-class and
between-class changes remain strongly separated; we use this as a comparison
relative to the control, not as an additive causal correction. The same sign
pattern appears in the wider architectures of Appendix~\ref{app:wisconsin},
which were run on a single exploratory split. The change also grows with depth
within a network: in $(64,64,64)$ the within-malignant change is $-7.7\%$ in the
first hidden layer and $-48.5\%$ in the third, while the between-class change
grows from $+7.6\%$ to $+30.5\%$ (Table~\ref{tab:wis-layer-main}). The first
layer of the control is flat, $0.0\pm0.6$ per cent, which is the quantity the
layerwise theory predicts and which no global constant would reproduce.

\begin{table*}[t]
\centering
\caption{Layerwise relative change for the $(64,64,64)$ network, over the same
$10\times5$ split--seed design as Table~\ref{tab:wis-main}. Entries are the mean
and one standard deviation across the $50$ combinations.}
\label{tab:wis-layer-main}
\small
\begin{tabular}{lrrr}
\toprule
segment family & layer 1 & layer 2 & layer 3\\
\midrule
within malignant & $-7.7\pm4.4$ & $-27.1\pm11.1$ & $-48.5\pm12.4$\\
within benign & $-7.6\pm3.8$ & $-26.1\pm6.5$ & $-45.1\pm11.0$\\
between class & $+7.6\pm2.8$ & $+10.1\pm6.8$ & $+30.5\pm12.1$\\
off manifold (control) & $+0.0\pm0.6$ & $-10.4\pm4.1$ & $-17.7\pm6.8$\\
\bottomrule
\end{tabular}
\end{table*}

The separation develops monotonically and is close to saturation by epochs $150$--$300$; Figure~\ref{fig:wis-splits}(b) follows one continuous Adam run through epoch $600$. Across all nine combinations of $\sigma_b^2\in\{0.01,0.1,1\}$ and training length in $\{100,300,1000\}$, the sign pattern persists. At $\sigma_b^2=1$, for example, the effect shifts from removing within-class boundaries toward adding between-class boundaries, but the ordering remains the same. Appendix~\ref{app:wisconsin} reports these robustness runs and the additional architectures in full.

Finally, the number of scalar output kinks equals the total hidden switch count
in $159{,}094$ of the $159{,}100$ segment measurements of the replicated study,
and in every exploratory run of Appendix~\ref{app:wisconsin}. The six exceptions
all occur in one trained network, where the slope change at the switch has
relative size about $10^{-9}$ and therefore falls below the tolerance of the
slope-comparison test; no switch is invisible at the output. At initialization
this agreement is what the visibility theorem predicts, and at the tested widths
it is stronger than required by it. After training it is a separate empirical observation. Together, the initialization agreement, class-dependent reorganization, depth decomposition, and training trajectory indicate that the statistic is tracking learned geometry rather than merely a global rescaling of the network.

\section{Discussion and limitations}
The theoretical and empirical parts answer two sides of the same question. At random initialization, conditional Kac--Rice and covariance propagation make the finite-network affine geometry explicit. On real data, the same statistic is accurately predicted before training and then reorganized in a label-dependent way by learning. The observed effect is not a generic increase in complexity: within-class paths become geometrically simpler while between-class paths acquire more boundaries.

The trained-network observation is consistent with, but distinct from, recent results on local linear-region complexity. \citet{HumayunBalestrieroBaraniuk2024} report that linear regions migrate away from data points toward the decision boundary during a late phase of training, while \citet{PatelMontufar2025} relate lower local complexity to feature learning and optimization. Our chord-wise switch count is neither a neighborhood-based local-complexity measure nor a decision-boundary density. It records the hidden activation switches encountered along a prescribed segment and, crucially for the present comparison, has an explicit Gaussian initialization benchmark. The opposite changes on within-class and between-class segments therefore provide a class-conditioned one-dimensional diagnostic of boundary redistribution rather than a global complexity measure.

The results concern expectations at fixed depth and fixed input dimension. We do not prove concentration, rates for the $o(N)$ terms, growing-dimensional limits, or post-training analogues of the Gaussian theorems. In $d\ge2$, boundary measure does not determine the number of full-dimensional affine regions. The off-manifold control addresses a simple global-contraction explanation but is not an exact geometry-matched counterfactual. The principal experiment varies both the train/test split and the network
initialization, but the splits are re-partitions of a single dataset rather than
independent datasets, so the reported variability measures sensitivity to
partition and initialization and not sampling variability of the data source.
The wider architectures and the bias-variance sweep remain single-split
exploratory checks, and the study uses one dataset in one application domain. Full proofs, synthetic tables, Crofton calibration, unequal-width checks, all additional Wisconsin architectures, and robustness details appear in the appendices.

\appendix
\numberwithin{equation}{section}
\numberwithin{table}{section}
\numberwithin{figure}{section}
\section*{Appendices}
The appendices give the complete finite-network arguments and the broader numerical checks. The main text states the scalar-output results; several proofs below establish stronger auxiliary hidden-switch or hidden-boundary estimates on the way to those conclusions.
\section{Proof of Theorem \ref{thm:affine-region-asymptotics}:  1D theorem}
\label{sec:proof-affine-region-asymptotics}

This section reduces Theorem~\ref{thm:affine-region-asymptotics} to the expected hidden-switch count and a visibility estimate. We first record the finite Gaussian geometry needed to identify activation switches with the breakpoints of the activation partition and to determine when a switch reaches the scalar output. We then state the two central propositions, derive the theorem from them immediately, and prove the propositions in the subsequent sections. All switch counts are taken in the open interval $(a,b)$; endpoint zeros are ignored because all relevant endpoint zero probabilities are zero. Let $\cF_\ell$ denote the sigma-field generated by all weights and biases up to hidden layer $\ell$.

\begin{definition}[Switches, finite partitions, and visibility]
\label{def:switch-notions}
Let $I=[a,b]\subset\R$ be a compact interval.
\begin{enumerate}[label=\textup{(\roman*)},leftmargin=2.5em]
\item A point $t_0\in(a,b)$ is an \emph{activation switch} of neuron $(\ell,j)$ if $z_{\ell,j}(t_0)=0$ and $z_{\ell,j}$ changes sign at $t_0$.
\item The \emph{switch counts} are
\[
        N_{\ell,j}(I):=\#\{t_0\in(a,b):\ t_0\text{ is an activation switch of }z_{\ell,j}\},
        \qquad
        S_N(I):=\sum_{\ell=1}^L\sum_{j=1}^{n_\ell}N_{\ell,j}(I).
\]
\item The \emph{activation pattern} at $t$ is $A_N(t)=(\1\{z_{\ell,j}(t)>0\})_{1\le\ell\le L,\,1\le j\le n_\ell}$, and the \emph{activation partition} $\cP_N(I)$ is the partition of $I$ into maximal intervals on which $A_N(t)$ is constant.
\item An activation switch is \emph{visible} if it creates a genuine kink of the scalar output, that is, if $f'_{N,+}(t_0)\ne f'_{N,-}(t_0)$.  The number of visible activation switches in $(a,b)$ is denoted by $V_N(I)$.
\item Put $\mathcal B_{0,N}=\varnothing$ and let $\mathcal B_{\ell,N}$ be the union of $\mathcal B_{\ell-1,N}$ and the switch locations of all neurons in layer $\ell$.  The connected components of $(a,b)\setminus\mathcal B_{\ell,N}$ are denoted by $\mathcal C_{\ell,N}$.
\item If $t_0$ is an isolated switch of $u=(\ell,j)$ and no other hidden preactivation vanishes at $t_0$, its \emph{downstream sensitivity} is the coefficient of $h_{\ell,j}$ in the scalar output when all downstream gates are frozen at their values at $t_0$:
\[
        D_u(t_0):=\frac{\partial f_N(t_0)}{\partial h_{\ell,j}(t_0)}.
\]
A path from $u$ to the output is \emph{active} at $t_0$ when all its downstream hidden neurons are active there.
\end{enumerate}
\end{definition}

For a continuous piecewise-affine function, an unadorned derivative denotes the ordinary derivative where it exists and is set equal to zero at the finitely many points where it does not exist.  This convention affects neither Lebesgue integrals nor any fixed input probabilistic identity below.  One-sided derivatives are always denoted explicitly. We now derive the  Theorem~\ref{thm:affine-region-asymptotics} from a sequence of auxiliary results. We state those auxiliary results now and postpone their proof to the end of this section. First, the finite-width argument needs a genericity statement that separates switch locations and excludes cancellation along an active downstream path. This statement also ensures that the activation partition can be described by finitely many isolated switches.

\begin{lemma}[Finite Gaussian geometry]
\label{lem:finite-gaussian-geometry}
For every fixed finite architecture, there is an event of probability one on which, simultaneously for all layers, the following properties hold.
\begin{enumerate}[label=\textup{(\roman*)},leftmargin=2.5em]
\item Each set $\mathcal B_{\ell,N}$ is finite, and every hidden activation is affine on each cell in $\mathcal C_{\ell,N}$.
\item No layer-$\ell$ preactivation vanishes at a point of $\mathcal B_{\ell-1,N}$, every zero in a cell of $\mathcal C_{\ell-1,N}$ is isolated and sign-changing, and no two distinct neurons switch at the same input point.
\item If a switch has an active downstream path to the output, then its downstream sensitivity is nonzero.
\end{enumerate}
\end{lemma}

The preceding geometry converts switch locations into exact identities for the activation partition and the scalar kink set. These identities isolate visibility as the only remaining difference between switches and scalar kinks.

\begin{lemma}[Switch counts and scalar kinks]
\label{lem:partition-identities}
On the event in Lemma~\ref{lem:finite-gaussian-geometry}, we have
	\(\#\cP_N(I)=1+S_N(I),\) \(R_N(I)=V_N(I),\) and
\(
 f'_{N,+}(t_0)-f'_{N,-}(t_0)
 =D_u(t_0)\bigl(h'_{\ell,j,+}(t_0)-h'_{\ell,j,-}(t_0)\bigr)
\)
for every switch $t_0$ of $u=(\ell,j)$.
\end{lemma}

The theorem now follows from the partition identity and two central propositions. The first proposition, whose proof will be given in the next section, gives the expected number of switches, while the second shows that switches that are not visible at the scalar output make a negligible contribution.

\begin{proposition}[Expected number of activation switches]
\label{prop:wide-switch-asymptotics}
In the centered Gaussian fully connected model of Section~\ref{sec:main-asymptotic-theorem}, assume fixed hidden depth $L$ and the proportional-width regime defined there. Then, for every compact interval $I=[a,b]$,
\[
        \E S_N(I)
        =\sum_{\ell=1}^L n_\ell\int_I\rho_\ell(t)\,\dd t+o(N).
\]
\end{proposition}

\begin{proposition}[Negligibility of invisible activation switches]
\label{prop:asymptotic-visibility}
Assume the centered Gaussian fully connected model of Section~\ref{sec:main-asymptotic-theorem}.  If $\E S_N(I)=O(N)$ and $N\sum_{r=1}^L2^{-n_r}\to0$, then $\E[S_N(I)-V_N(I)]\to0$.
\end{proposition}
\begin{proof}[Proof of Theorem~\ref{thm:affine-region-asymptotics}]
By Proposition~\ref{prop:wide-switch-asymptotics}, we have
\(
        \E S_N(I)=\sum_{\ell=1}^L n_\ell\int_I\rho_\ell(t)\,\dd t+o(N).
\)
In particular, $\E S_N(I)=O(N)$.  Since the widths are proportional to $N$, we have $N\sum_{r=1}^L2^{-n_r}\to0$.  Therefore, Proposition~\ref{prop:asymptotic-visibility} gives $\E[S_N(I)-V_N(I)]\to0$.

Lemma~\ref{lem:partition-identities} gives $R_N(I)=V_N(I)$ almost surely. Taking expectations gives
\(
        \E R_N(I)=\E V_N(I)=\E S_N(I)+o(1).
\)
This gives the stated $o(N)$ asymptotic after inserting the asymptotic formula for activation switches.  Dividing by $N$ and using $n_\ell/N\to\alpha_\ell$ gives the equivalent normalized limit.
\end{proof}

\begin{proof}[Proof of Lemma~\ref{lem:partition-identities}]
Between two consecutive points of $\mathcal B_{L,N}$ every gate is constant, so every hidden activation and the scalar output are affine.  At a point of $\mathcal B_{L,N}$ exactly one preactivation changes sign.  Crossing the point therefore changes exactly one entry of the activation vector, and a finite set of $S_N(I)$ distinct interior switch locations divides $I$ into $1+S_N(I)$ activation intervals.
Near a switch of $u$, all other gates are fixed.  Repeated substitution through the downstream affine maps makes the scalar output an affine function of $h_{\ell,j}$ with coefficient $D_u(t_0)$, plus terms whose derivatives agree on the two sides.  Taking one-sided derivatives gives the displayed jump identity.  The local derivative jump of $h_{\ell,j}$ is nonzero because the zero is sign-changing with nonzero slope.  Finally, a continuous piecewise-affine scalar function has a maximal affine boundary precisely where its two one-sided derivatives differ.  Hence, the number of its affine regions is one plus the number of visible switches.
\end{proof}

\begin{remark}[Activation switches and inherited kinks]
\label{rem:switch-vs-kink}
Figure~\ref{fig:switch-vs-kink} separates four facts used below. The
preactivation $z_{2,1}$ inherits kinks from the first layer at $t=-2,0,1$, but
it does not vanish there. The finite Gaussian geometry established below shows
that inherited kink locations are almost surely not zeros of a later
preactivation. Thus, $N_{2,1}(I)$ counts only the new sign-changing zeros
$t=-1.5,0.75,1.4$, and each switch in $S_N(I)$ belongs to the unique neuron
whose preactivation vanishes.

Every kink of the scalar output nevertheless originates from a switch in some
hidden neuron.  In panel (c), the kink of $h_{2,1}$ at $t=0$ is inherited from
the switch of $h_{1,2}$ and passes through an open gate.  Affine combinations
create no new breakpoints, while a ReLU creates a new breakpoint only when its
argument changes sign.  Hence, every hidden kink, and therefore every kink of
$f_N$, occurs at a switch.  At a new sign-changing zero of $z_{2,1}$,
the one-sided derivatives of $h_{2,1}$ differ.  Whether this local kink reaches
the scalar output is determined by the downstream sensitivity in
Definition~\ref{def:switch-notions}. Closed gates erase inherited kinks, as at
$t=-2$ and $t=1$ in the figure; the visibility estimate below shows that the
expected number of switches erased in this way is negligible as the widths
grow.
\end{remark}

\begin{proof}[Proof of Proposition~\ref{prop:asymptotic-visibility}]
Fix $1\le\ell\le L-1$ and condition on $\cF_\ell$.  On the event of Lemma~\ref{lem:finite-gaussian-geometry}, order the layer-$\ell$ switches increasingly by input location, using the neuron index only as a deterministic rule for breaking ties.  This gives an $\cF_\ell$-measurable list $T_1,\dots,T_M$, where $M=\sum_jN_{\ell,j}(I)$.
Fix $T_m$ and let $A_{r,m}$ be the event that all neurons in a downstream layer $r>\ell$ are inactive at $T_m$.  Since $T_m$ is $\cF_{r-1}$-measurable, conditional on $\cF_{r-1}$ the variables $z_{r,k}(T_m)$, $1\le k\le n_r$, are independent centered Gaussians with conditional variances at least $\beta_r$.  Their signs are therefore independent and fair, so
\(
 \Pp(A_{r,m}\mid\cF_{r-1})=2^{-n_r}.
\)
Since $\cF_\ell\subseteq\cF_{r-1}$, the tower property then gives
\(
 \Pp(A_{r,m}\mid\cF_\ell)
 =\E[\Pp(A_{r,m}\mid\cF_{r-1})\mid\cF_\ell]
 =2^{-n_r}.
\)
If none of the events $A_{r,m}$ occurs, choose one active neuron in each downstream hidden layer.  Full connectivity joins these neurons to the switching neuron and the output, giving an active path.  Lemma~\ref{lem:finite-gaussian-geometry} makes its downstream sensitivity nonzero, and Lemma~\ref{lem:partition-identities} then makes the switch visible.  Consequently, an invisible switch is contained, up to a null set, in $\bigcup_{r=\ell+1}^L A_{r,m}$.  The conditional union bound and conditional linearity over the measurable finite list give
\[
        \E\!\big[
        \#\{\text{invisible activation switches in layer }\ell\}
        \mid \cF_\ell
        \big]
        \le
        \sum_{r=\ell+1}^{L}2^{-n_r}
        \sum_{j=1}^{n_\ell}N_{\ell,j}(I).
\]
For $\ell=L$, the direct edge to the output is active and its coefficient is $a_j/\sqrt{n_L}\ne0$ almost surely, so every switch in the last layer is visible.  Summing the displayed estimate over $\ell=1,\dots,L-1$ and taking expectations yields
\[
        \E[S_N(I)-V_N(I)]
        \le
        \sum_{\ell=1}^{L-1}
        \sum_{r=\ell+1}^{L}2^{-n_r}
        \E\sum_{j=1}^{n_\ell}N_{\ell,j}(I)
        \le
        \sum_{r=1}^L2^{-n_r}\E S_N(I).
\]
By the assumptions of Proposition~\ref{prop:asymptotic-visibility}, the right side tends to zero.
\end{proof}

\begin{proof}[Proof of Lemma~\ref{lem:finite-gaussian-geometry}]
We first prove (i) and (ii) by induction over the layers.  In layer one, $z_{1,j}(t)=b_{1,j}+w_{1,j}t$.  Since $(b_{1,j},w_{1,j})$ has a nondegenerate density, $w_{1,j}\ne0$ almost surely, each neuron has at most one isolated sign-changing zero, and the determinant associated with a common root $b_{1,j}w_{1,k}-b_{1,k}w_{1,j}$ is a nontrivial polynomial for $j\ne k$.  Thus, two first-layer neurons have a common zero with probability zero.

Suppose the result holds through layer $\ell-1$ and condition on $\cF_{\ell-1}$.  The set $\mathcal B_{\ell-1,N}$ is finite.  On a cell $C\in\mathcal C_{\ell-1,N}$, write $h_{\ell-1,i}(t)=p_{i,C}+q_{i,C}t$.  The restriction of a new preactivation is
\[
 z_{\ell,j}(t)=U_{j,C}+V_{j,C}t,
 \quad
 \Var(V_{j,C}\mid\cF_{\ell-1})
 =\frac{\gamma_\ell}{n_{\ell-1}}\sum_iq_{i,C}^2=:Q_C.
\]
If $Q_C=0$, then $V_{j,C}=0$ almost surely, while $U_{j,C}$ has conditional variance at least $\beta_\ell$; hence, the restriction has no zero almost surely.  If $Q_C>0$, the independent bias gives
\(
 \det\Cov((U_{j,C},V_{j,C})\mid\cF_{\ell-1})
 \ge \beta_\ell Q_C>0.
\)
Consequently, the restriction is not identically zero and any root in $C$ has nonzero slope.  For two neurons in the current layer, the condition for a common root is $U_{j,C}V_{k,C}-U_{k,C}V_{j,C}=0$.  When $Q_C>0$ this determinant is a nontrivial polynomial in two independent nondegenerate Gaussian pairs, so it vanishes with probability zero.  At any $s\in\mathcal B_{\ell-1,N}$, the conditional variance of $z_{\ell,j}(s)$ is at least $\beta_\ell$, and hence $z_{\ell,j}(s)\ne0$ almost surely.  A finite union over cells, neurons, and boundary points proves the induction.  In particular, each neuron in the current layer adds at most one switch per old cell.  On every resulting subcell the sign of each new preactivation is constant, so its ReLU activation is either zero or the same affine function.  This also proves the affine assertion in (i).

It remains to prove (iii).  Fix a layer-$\ell$ switch $t_0$ of $u=(\ell,j)$ and condition on $\cF_\ell$; by ordering the finite switch set first by location and then by neuron index, this conditioning can be made simultaneously for every switch.  The downstream weights, biases, and output weights form a finite-dimensional Gaussian vector with a density.  Successive conditioning on the preceding layer shows that no downstream preactivation equals zero at $t_0$ almost surely.  Outside the null sets on which a downstream preactivation vanishes, partition the parameter space according to the complete downstream gate vector $g$.  On a fixed gate region $\Omega_g$, the identities $h_{r,k}(t_0)=g_{r,k}z_{r,k}(t_0)$ show recursively that all preactivations at $t_0$ are polynomials in the downstream parameters.  Thus, $\Omega_g$ is described by finitely many strict polynomial inequalities.

On $\Omega_g$,  we have
\(
 D_u(t_0)=\sum_{\pi:u\to {\rm out}}
 \Big(\prod_{e\in\pi}c_e w_e\Big)
 \Big(\prod_{v\in\pi\setminus\{u,{\rm out}\}}g_v\Big),
\)
where the positive constants $c_e$ are the width normalizations.  This is a polynomial.  If $g$ contains an active path, choose one such path.  Its edge monomial occurs with a nonzero coefficient and cannot be produced by a different path, since two distinct paths use different edge variables at their first divergence.  The polynomial is therefore nontrivial.  Its zero set has Lebesgue measure zero, and hence Gaussian probability zero.  There are finitely many gate vectors and finitely many switches.  Taking the corresponding finite unions proves (iii).
\end{proof}

\section{Proof of Proposition \ref{prop:wide-switch-asymptotics}: expected activation-switch count}
\label{sec:exact-conditional-kac-rice}
This section proves Proposition~\ref{prop:wide-switch-asymptotics} from three inputs. The conditional Kac--Rice formula gives an exact finite-width expression for one neuron's switch count, uniform fourth moments provide the integrability needed to average that expression, and propagation of empirical ReLU moments identifies the deterministic covariance limit. We state all three inputs first and then derive the proposition before proving them. The conditional Kac--Rice proposition itself is reduced afterward to an elementary piecewise-affine Gaussian zero formula.
For $2\le\ell\le L$ and $1\le j\le n_\ell$, define the conditional covariance quantities on the cells in $\mathcal C_{\ell-1,N}$ by
\[
\begin{aligned}
S_{\ell,N}(t)&:=\Var(z_{\ell,j}(t)\mid\cF_{\ell-1}),\\
Q_{\ell,N}(t)&:=\Var(z'_{\ell,j}(t)\mid\cF_{\ell-1}),\\
R_{\ell,N}(t)&:=\Cov(z_{\ell,j}(t),z'_{\ell,j}(t)\mid\cF_{\ell-1}).
\end{aligned}
\]

The next proposition is the exact finite-width input to the large-width switch-count calculation. It expresses the conditional expected switch count entirely through the quantities defined above.

\begin{proposition}[Conditional Kac--Rice formula at finite width]
\label{prop:conditional-kac-rice}
Fix $2\le\ell\le L$ and $1\le j\le n_\ell$.  Conditional on $\cF_{\ell-1}$,
\[
        \E\bigl[N_{\ell,j}(I)\mid\cF_{\ell-1}\bigr]
        =\frac1\pi\int_I
       \sqrt{ 
        \frac{Q_{\ell,N}(t)}{S_{\ell,N}(t)}-
        \Big(\frac{R_{\ell,N}(t)}{S_{\ell,N}(t)}\Big)^2
        }\dd t.
\]
For $\ell=1$, the corresponding formula is $\E N_{1,j}(I)=\int_I\sqrt{\beta_1\gamma_1}/[\pi(\beta_1+\gamma_1t^2)]\,\dd t$.
\end{proposition}

We recall the notation $A_\ell,B_\ell$ and $\rho_\ell$ defined in Section~\ref{sec:main-asymptotic-theorem}. For every fixed deterministic $t$, Lemma~\ref{lem:finite-gaussian-geometry} and nondegeneracy of the preactivations imply that $t$ is an activation boundary with probability zero. Consequently, our derivative convention satisfies $h'_{\ell,j}(t)=z'_{\ell,j}(t)\1\{z_{\ell,j}(t)>0\}$ almost surely. Put $\kappa_{1,N}(t):=\rho_1(t)$, and for $\ell\ge2$ set, off $\mathcal B_{\ell-1,N}$,
\[
        \kappa_{\ell,N}(t)
        :=\frac1\pi
       \sqrt{ 
        \frac{Q_{\ell,N}(t)}{S_{\ell,N}(t)}-
        \Big(\frac{R_{\ell,N}(t)}{S_{\ell,N}(t)}\Big)^2
        }.
\]
Set $\kappa_{\ell,N}=0$ on $\mathcal B_{\ell-1,N}$; these finitely many values are immaterial for all integrals below.
The second input controls the tails of the preactivations and their derivatives uniformly on the compact interval. It supplies the uniform integrability needed when the random Kac--Rice density is averaged.

\begin{lemma}[Moment bounds]
\label{lem:uniform-fourth-moments}
For every fixed layer $\ell$ and every compact interval $I$,
\[
        \sup_N\sup_{t\in I}
        \E\bigl[|z_{\ell,1}(t)|^4+|z'_{\ell,1}(t)|^4\bigr]<\infty.
\]
	Consequently, $\E[h_{\ell,1}(t)^4]$ and $\E[h'_{\ell,1}(t)^4]$, as well as $\E[h_{\ell,1}(t)^2h'_{\ell,1}(t)^2]$, are uniformly bounded in $N$ and $t\in I$.
\end{lemma}

The third input identifies the deterministic covariance data that enter the Kac--Rice density. For fixed $t\in I$, write $M_{\ell,N}(t)$ for the vector of empirical ReLU moments
\[
        M_{\ell,N}(t):=\biggl(\frac1{n_\ell}\sum_jh_{\ell,j}(t)^2,
        \frac1{n_\ell}\sum_jh'_{\ell,j}(t)^2,
        \frac1{n_\ell}\sum_jh_{\ell,j}(t)h'_{\ell,j}(t)\biggr).
\]

\begin{proposition}[Propagation of covariance moments]
\label{prop:propagation-covariance-moments}
Fix $t\in I$.  Under the regime of fixed depth and proportional widths, for every $\ell\ge1$, in probability and in $L^1$,
\[
        M_{\ell,N}(t)\to
        \Bigl(\frac12(A_\ell+B_\ell t^2),\frac12B_\ell,\frac12B_\ell t\Bigr),
        \qquad
        (S_{\ell,N}(t),Q_{\ell,N}(t),R_{\ell,N}(t))\to
        (A_\ell+B_\ell t^2,B_\ell,B_\ell t)
\]
for the covariance vector whenever $\ell\ge2$.
\end{proposition}

The three stated inputs now determine the switch-count asymptotic. The Kac--Rice identity reduces each layer to the expectation of its random intensity, the covariance limit identifies the pointwise deterministic density, and the moment bound supplies uniform integrability and dominated convergence.

\begin{proof}[Proof of Proposition~\ref{prop:wide-switch-asymptotics}]
Fix a layer $\ell$.  Proposition~\ref{prop:conditional-kac-rice} gives
$\E N_{\ell,1}(I)=\int_I\E\kappa_{\ell,N}(t)\,\dd t$; for $\ell=1$ this
uses $\kappa_{1,N}=\rho_1$, and for $\ell\ge2$ it follows by first
conditioning on $\cF_{\ell-1}$.  Proposition~\ref{prop:propagation-covariance-moments}
implies, for each fixed $t$, that the covariance triple converges in
probability to $(A_\ell+B_\ell t^2,B_\ell,B_\ell t)$.  Since its first
coordinate is positive, continuity of the Kac--Rice integrand gives
\[
 \kappa_{\ell,N}(t)\longrightarrow
	\frac1\pi\sqrt{
 \frac{B_\ell}{A_\ell+B_\ell t^2}
 -\frac{B_\ell^2t^2}{(A_\ell+B_\ell t^2)^2}
 }=\rho_\ell(t)
\]
in probability.
It remains to justify passage to expectations and integration. For
$\ell\ge2$, the positive bias variance and Lemma~\ref{lem:uniform-fourth-moments}
give the quantitative chain
	\( \E\kappa_{\ell,N}(t)^2
 \le \frac{\E Q_{\ell,N}(t)}{\pi^2\beta_\ell}
 =\frac{\gamma_\ell}{\pi^2\beta_\ell}
   \E h'_{\ell-1,1}(t)^2
 \le C_\ell,\)
\[
\begin{aligned}
 \kappa_{\ell,N}(t)\xrightarrow{\Pp}\rho_\ell(t)
 \ \Longrightarrow\ 
 \E\kappa_{\ell,N}(t)\to\rho_\ell(t),\quad\text{ and }\quad
 \E N_{\ell,1}(I)
 &=\int_I\E\kappa_{\ell,N}(t)\,\dd t
 \longrightarrow\int_I\rho_\ell(t)\,\dd t.
\end{aligned}
\]
The second line uses uniform integrability, and the third uses dominated
convergence; the same $L^2$ bound gives a deterministic integrable bound for
$\E\kappa_{\ell,N}(t)$ on the compact interval $I$. For $\ell=1$, the identity
is exact. Exchangeability within each layer and fixed $L$ therefore give
\(
 \E S_N(I)
 =\sum_{\ell=1}^L n_\ell\E N_{\ell,1}(I)
 =\sum_{\ell=1}^L n_\ell\int_I\rho_\ell(t)\,\dd t+o(N).
\)
\end{proof}

We next prove the three inputs in the order in which their additional arguments arise. The conditional Kac--Rice proposition uses one auxiliary zero-count identity for a deterministic piecewise-affine Gaussian field. The remaining two inputs are then proved by moment induction and conditional laws of large numbers.

Its proof reduces to a deterministic piecewise-affine Gaussian field. Let $f_1,\dots,f_m$ be deterministic continuous piecewise-affine functions on $I=[a,b]$ with a common finite affine partition, and let $\zeta,\xi_1,\dots,\xi_m$ be independent centered Gaussian variables with $\Var(\zeta)=\tau^2>0$ and $\Var(\xi_i)=v_i^2$. Write $Z(t)=\zeta+\sum_{i=1}^m\xi_i f_i(t)$ and, at differentiability points, define
	\(S(t)=\tau^2+\sum_i v_i^2f_i(t)^2,\) 
	\(Q(t)=\sum_i v_i^2f_i'(t)^2,\) and
\(
        R(t)=\sum_i v_i^2f_i(t)f_i'(t).
\)
The auxiliary formula below computes the expected number of zeros from these quantities.

\begin{lemma}[Piecewise-affine Gaussian zero formula]
\label{lem:piecewise-kac-rice}
Under the preceding setup, we have
\[
	\E\big[\#\{t\in(a,b):Z(t)=0\}\big]
        =\frac1\pi\int_I
	\sqrt{ 
        \frac{Q(t)}{S(t)}-
        \Big(\frac{R(t)}{S(t)}\Big)^2
        }\dd t.
\]
\end{lemma}

\begin{proof}[Proof of Proposition~\ref{prop:conditional-kac-rice}]
Conditional on $\cF_{\ell-1}$, Lemma~\ref{lem:finite-gaussian-geometry} gives a finite partition on which the functions $h_{\ell-1,i}$ are deterministic and affine.  Moreover
\[
        z_{\ell,j}(t)
        =b_{\ell,j}
        +\sum_{i=1}^{n_{\ell-1}}\frac{W_{\ell,ji}}{\sqrt{n_{\ell-1}}}h_{\ell-1,i}(t)
\]
is a centered Gaussian process in the variables from the current layer $(b_{\ell,j},W_{\ell,j1},\dots,W_{\ell,jn_{\ell-1}})$, with the bias variance $\beta_\ell$ playing the role of $\tau^2$.  Lemma~\ref{lem:piecewise-kac-rice}, applied conditionally to this finite random partition, gives the displayed integral.  On a cell with $Q_{\ell,N}=0$ the restriction is a nonzero random constant almost surely.  On a cell with $Q_{\ell,N}>0$ the affine coefficient pair is nondegenerate, so every zero has nonzero slope and is sign-changing.  At each point of $\mathcal B_{\ell-1,N}$ the conditional variance of $z_{\ell,j}$ is at least $\beta_\ell$, and hence the probability of a boundary zero is zero.  The finite conditional zero count therefore equals $N_{\ell,j}(I)$ almost surely.

For $\ell=1$, write $b_{1,j}=\sqrt{\beta_1}X$ and $W_{1,j}=\sqrt{\gamma_1}Y$, where $X$ and $Y$ are independent standard Gaussian variables. The unique zero of $b_{1,j}+W_{1,j}t$ is $T=-\sqrt{\beta_1/\gamma_1}\,X/Y$. Since the ratio $X/Y$ has the standard Cauchy law, $T$ has density $\sqrt{\beta_1\gamma_1}/[\pi(\beta_1+\gamma_1t^2)]$. The switch count $N_{1,j}(I)$ is the indicator that $T$ lies in the interior of $I$, up to the null event $W_{1,j}=0$ and the endpoint null events. Integrating this density over $I$ gives the first-layer formula.
\end{proof}

\begin{proof}[Proof of Lemma~\ref{lem:piecewise-kac-rice}]
Let $a=s_0<\cdots<s_K=b$ be a common affine partition. On $C_k=(s_{k-1},s_k)$ write $Z(t)=U_k+V_kt$. Since $\zeta$ enters $U_k$ with coefficient one and does not enter $V_k$, the Cauchy--Schwarz inequality gives, for $t\in C_k$,
\[
 S(t)Q(t)-R(t)^2
 =\tau^2Q(t)
 +\left(\sum_i v_i^2f_i(t)^2\right)
  \left(\sum_i v_i^2f_i'(t)^2\right)
 -\left(\sum_i v_i^2f_i(t)f_i'(t)\right)^2
 \ge \tau^2Q(t).
\]
Here, $Q(t)=\Var(V_k)$ is constant on $C_k$. If $Q(t)=0$, then $V_k=0$ almost surely, whereas $U_k$ has variance at least $\tau^2$; the cell therefore contains no zero almost surely. If $Q(t)>0$, the displayed inequality shows that $(U_k,V_k)$ is nondegenerate and the affine restriction has at most one zero. A zero at $t\in C_k$ occurs exactly when $U_k=-tV_k$. 
Thus, the expected zero count can be obtained by integrating the joint density of $(U_k,V_k)$ over the pairs $(u,v)$ for which $u+vt=0$ for some $t\in C_k$.

Parameterize these pairs by $(t,v)$ through $(u,v)=(-tv,v)$. The absolute value of the Jacobian of this parameterization is $|v|$. Hence,
\(
\E\big[\#\{t\in C_k:Z(t)=0\}\big]
=\int_{C_k}\int_{\R}|v|p_{U_k,V_k}(-tv,v)\,\dd v\dd t.
\)
For fixed $t$, the linear change of variables $(u,v)\mapsto(z,v)=(u+tv,v)$ has Jacobian one, so
$p_{U_k,V_k}(-tv,v)=p_{Z(t),V_k}(0,v)$. Factoring the latter joint density into the marginal density of $Z(t)$ and the conditional density of $V_k$ given $Z(t)=0$, we therefore obtain
\(
\E\big[\#\{t\in C_k:Z(t)=0\}\big]
=\int_{C_k}p_{Z(t)}(0)
\E\big[|V_k|\mid Z(t)=0\big]\,\dd t.
\)
This is the Kac--Rice identity in the present affine setting. The factor $p_{Z(t)}(0)$ measures how much probability mass the field places near zero at $t$, while the conditional factor $\E[|V_k|\mid Z(t)=0]$ weights a zero by the speed at which the affine restriction crosses zero.

Since $(Z(t),V_k)$ is a centered Gaussian pair with variances $S(t)$ and $Q(t)$ and covariance $R(t)$, the conditional law of $V_k$ given $Z(t)=0$ is centered Gaussian with variance $Q(t)-R(t)^2/S(t)$. Hence,
\[
 p_{Z(t)}(0)=\frac1{\sqrt{2\pi S(t)}},\qquad
 \E[|V_k|\mid Z(t)=0]
	=\sqrt{\frac2\pi}\,\sqrt{Q(t)-R(t)^2/S(t)},
\]
and their product is $\pi^{-1}[Q(t)/S(t)-(R(t)/S(t))^2]^{1/2}$. Finally, $Z(s_k)$ has variance at least $\tau^2$ at every partition point, so none of the partition points is a zero almost surely. Summing over the cells proves the formula. The same argument excludes an identically zero restriction: for $Q>0$ the pair $(U_k,V_k)$ has a density.
\end{proof}

We now prove the auxiliary estimates.

\begin{proof}[Proof of Lemma~\ref{lem:uniform-fourth-moments}]
The claim is immediate in the first layer because
$z_{1,1}(t)=b_{1,1}+w_{1,1}t$ and $z'_{1,1}(t)=w_{1,1}$ are Gaussian with
variances bounded uniformly for $|t|\le\sup_{t\in I}|t|$.
Assume the claim in layer $\ell-1$ and fix $t\in I$.  Outside a null event at
this fixed input, the conditional variances and fourth moments satisfy
	\(Q_{\ell,N}(t)=\frac{\gamma_\ell}{n_{\ell-1}} \sum_i h'_{\ell-1,i}(t)^2,\)
\[
\begin{aligned}
S_{\ell,N}(t)&=\beta_\ell+\frac{\gamma_\ell}{n_{\ell-1}}
 \sum_i h_{\ell-1,i}(t)^2,\\
\E[z_{\ell,1}(t)^4\mid\cF_{\ell-1}]&=3S_{\ell,N}(t)^2,\qquad
\E[z'_{\ell,1}(t)^4\mid\cF_{\ell-1}]=3Q_{\ell,N}(t)^2.
\end{aligned}
\]
Jensen's inequality gives
\(
 \Big(\frac1{n_{\ell-1}}\sum_i h_{\ell-1,i}(t)^2\Big)^2
 \le\frac1{n_{\ell-1}}\sum_i h_{\ell-1,i}(t)^4,
\)
and the same estimate holds for the derivatives.  Exchangeability and the
induction hypothesis therefore bound the expectations of $S_{\ell,N}(t)^2$
and $Q_{\ell,N}(t)^2$ uniformly in $N$ and $t$.  The inequalities
$|h_{\ell,1}|\le|z_{\ell,1}|$ and
$|h'_{\ell,1}|\le|z'_{\ell,1}|$ give the corresponding ReLU bounds, while
\(
 \E[h_{\ell,1}(t)^2h'_{\ell,1}(t)^2]
	\le\sqrt{\E h_{\ell,1}(t)^4}\sqrt{\E h'_{\ell,1}(t)^4)}
\)
gives the mixed bound.  Setting the derivative to zero at points of
nondifferentiability cannot increase any of these moments.
\end{proof}

\begin{proof}[Proof of Proposition~\ref{prop:propagation-covariance-moments}]
We argue by induction on $\ell$ for the fixed input $t$.  In the first layer,
we have 
	\( (z_{1,j}(t),z'_{1,j}(t))=(b_{1,j}+w_{1,j}t,w_{1,j}),\)
	\(\E[Z_+^2]=\tfrac12\Var(Z),\)
	\(\E[(Z')^2\1\{Z>0\}]=\tfrac12\Var(Z'),\) and 
\(
 \E[ZZ'\1\{Z>0\}]=\tfrac12\Cov(Z,Z')
\)
for a centered Gaussian pair $(Z,Z')$.  The ordinary law of large numbers,
together with Lemma~\ref{lem:uniform-fourth-moments}, gives the stated
convergence in $L^2$, and therefore in probability and $L^1$.
Suppose the empirical limits hold in layer $\ell-1$.  Conditional on
$\cF_{\ell-1}$, the pairs in layer $\ell$ are independent and identically distributed centered Gaussian, and
\[
 \begin{aligned}
 S_{\ell,N}(t)&=\beta_\ell+\frac{\gamma_\ell}{n_{\ell-1}}
 \sum_i h_{\ell-1,i}(t)^2,&
 \E[h_{\ell,j}(t)^2\mid\cF_{\ell-1}]&=\tfrac12S_{\ell,N}(t),\\
 Q_{\ell,N}(t)&=\frac{\gamma_\ell}{n_{\ell-1}}
 \sum_i h'_{\ell-1,i}(t)^2,&
 \E[h'_{\ell,j}(t)^2\mid\cF_{\ell-1}]&=\tfrac12Q_{\ell,N}(t),\\
 R_{\ell,N}(t)&=\frac{\gamma_\ell}{n_{\ell-1}}
 \sum_i h_{\ell-1,i}(t)h'_{\ell-1,i}(t),&
 \E[h_{\ell,j}(t)h'_{\ell,j}(t)\mid\cF_{\ell-1}]&=\tfrac12R_{\ell,N}(t).
 \end{aligned}
\]
The induction hypothesis and the recursions for $A_\ell$ and $B_\ell$ give
convergence of the covariance triple to
$(A_\ell+B_\ell t^2,B_\ell,B_\ell t)$ in probability and $L^1$. Conditional
independence and Lemma~\ref{lem:uniform-fourth-moments}, applied to the three
components of $M_{\ell,N}(t)$, yield
\[
\begin{aligned}
 \E\|M_{\ell,N}(t)-\E[M_{\ell,N}(t)\mid\cF_{\ell-1}]\|^2
 &\le \frac{C_{\ell,t}}{n_\ell},\\
 \E[M_{\ell,N}(t)\mid\cF_{\ell-1}]
 &=\frac12(S_{\ell,N}(t),Q_{\ell,N}(t),R_{\ell,N}(t)),\\
 \E\|M_{\ell,N}(t)-\tfrac12(A_\ell+B_\ell t^2,B_\ell,B_\ell t)\|
 &\le \frac{\sqrt{C_{\ell,t}}}{\sqrt{n_\ell}}
 +\frac12\E\|(S_{\ell,N},Q_{\ell,N},R_{\ell,N})\\
 &\hspace{4.2cm}-(A_\ell+B_\ell t^2,B_\ell,B_\ell t)\|\longrightarrow0.
\end{aligned}
\]
Hence, $M_{\ell,N}(t)\to\tfrac12(A_\ell+B_\ell t^2,B_\ell,B_\ell t)$ in
$L^1$ and therefore in probability. This completes the induction.
\end{proof}

\section{Proof of Theorem \ref{thm:higher-dimensional-boundaries}: higher dimensions}
\label{sec:higher-dimensional-boundaries}

We now prove Theorem~\ref{thm:higher-dimensional-boundaries} using the notation introduced in Section~\ref{subsec:higher-dimensional-main-result}. The argument has four inputs: finite polyhedral geometry, an exact conditional surface formula, propagation of the covariance matrices, and a visibility estimate for the scalar kink set. We state these four propositions first, derive the theorem immediately from them, and then prove the geometric and probabilistic inputs.
Let $\mathcal U$ be the finite set of hidden neurons, ordered first by layer
and then by neuron index, and let
$\varepsilon\in\{0,1\}^{\mathcal U}$. Replace each ReLU gate by multiplication
by the corresponding coordinate of $\varepsilon$, and define the resulting
affine functions $z_u^\varepsilon$ and $h_u^\varepsilon$ recursively. Thus,
$h_u^\varepsilon=\varepsilon_u z_u^\varepsilon$, and
$z_u^\varepsilon(x)=a_u^\varepsilon+
\langle g_u^\varepsilon,x\rangle$, where the coefficients are polynomials in
the finite parameter vector of the network. Set
\(
 C_\varepsilon
 =D\cap
 \{x:\min_{u\in\mathcal U} (2\varepsilon_u-1)z_u^\varepsilon(x)>0\},
 \) and \(
 \overline C_\varepsilon^{\,*}
 =D\cap
 \{x:\min_{u\in\mathcal U}(2\varepsilon_u-1)z_u^\varepsilon(x)\ge0\}.
\)
The second set is obtained by replacing the strict inequalities by weak ones;
since $D$ is open, it need not be the topological closure of $C_\varepsilon$.
If $\varepsilon_u=0$, let $\varepsilon^u$ be the mask obtained by changing
only the coordinate $u$ to one, and put
\(
 F_{\varepsilon,u}
 =D\cap\{z_u^\varepsilon=0\}\cap
 \{\min_{v\ne u}(2\varepsilon_v-1)z_v^\varepsilon>0\}.
\)
On $\{z_u^\varepsilon=0\}$, the affine functions associated with
$\varepsilon$ and $\varepsilon^u$ agree for every neuron other than $u$.
Every regular switching facet of the realized network therefore appears in
exactly one of the finitely many sets $F_{\varepsilon,u}$ with
$\varepsilon_u=0$. The first central proposition turns this deterministic mask representation into the generic finite polyhedral geometry used by the remaining arguments.
Let $\mathscr S_N$ denote the union of all faces of this masked polyhedral decomposition of dimension at most $d-2$, and write $u\prec v$ when $u$ belongs to an earlier hidden layer than $v$. For a mask $\varepsilon$ and hidden neuron $u$, let $D_u^\varepsilon$ denote the downstream sensitivity obtained by freezing all downstream gates according to $\varepsilon$.

\begin{proposition}[Finite polyhedral geometry and Gaussian genericity]
\label{lem:higher-dimensional-geometry}
For every fixed finite architecture, almost surely, the following assertions hold simultaneously for all masks $\varepsilon$ and hidden neurons $u,v$. Every nonempty cell $C_\varepsilon$ is full dimensional, so $\dim C_\varepsilon=d$, and the realized and masked networks agree there: $(z_w,h_w)|_{C_\varepsilon}=(z_w^\varepsilon,h_w^\varepsilon)|_{C_\varepsilon}$ for every hidden neuron $w$. No masked preactivation vanishes identically, that is, $(a_u^\varepsilon,g_u^\varepsilon)\ne(0,0)$, and every nonempty switching facet satisfies $g_u^\varepsilon\ne0$ and $F_{\varepsilon,u}\subseteq\{x\in D:a_u^\varepsilon+\langle g_u^\varepsilon,x\rangle=0\}$.

Distinct hidden-neuron boundaries meet only in codimension at least two: if $u\ne v$, then $\dim(\mathcal Z_u(D)\cap\mathcal Z_v(D))\le d-2$ and $\mathcal H^{d-1}(\mathcal Z_u(D)\cap\mathcal Z_v(D))=0$. Moreover, $\mathcal H^{d-1}(\mathscr S_N)=0$, and if $u\prec v$, then $\mathcal H^{d-1}(F_{\varepsilon,u}\cap\{z_v=0\})=0$. Finally, whenever $F_{\varepsilon,u}\ne\varnothing$ and $\varepsilon$ contains an active path from $u$ to the output, we have $D_u^\varepsilon\ne0$.
\end{proposition}

Recall from Section~\ref{subsec:higher-dimensional-main-result} that $\mathcal Z_{\ell,j}(D)=\mathcal Z_u(D)=\{x\in D:z_{\ell,j}(x)=0\}$ is the activation boundary of the neuron $u=(\ell,j)$, and that $\mathsf S_{N,d}(D)$ in~\eqref{eq:snd} sums the $\mathcal H^{d-1}$ measures of these boundaries over all hidden neurons.
At points where the preceding activations are differentiable, define
\( S_{\ell,N}(x)=\Var(z_{\ell,j}(x)\mid\cF_{\ell-1})\),
 \(R_{\ell,N}(x)=\Cov(z_{\ell,j}(x),\nabla z_{\ell,j}(x)\mid\cF_{\ell-1}),\)
 and 
\[
 \begin{aligned}
 C_{\ell,N}(x)&=\Cov(\nabla z_{\ell,j}(x)\mid\cF_{\ell-1}),&
 \Gamma_{\ell,N}(x)&=C_{\ell,N}(x)-
 \frac{R_{\ell,N}(x)R_{\ell,N}(x)^{\mathsf T}}{S_{\ell,N}(x)}.
 \end{aligned}
\]
Here $R_{\ell,N}(x)$ is a column vector. On the finite union of cell
boundaries, these quantities may be defined arbitrarily. Put
\[
 \kappa^{(d)}_{\ell,N}(x)
 =\frac{1}{\sqrt{2\pi S_{\ell,N}(x)}}
 \E\|G_{\ell,N}(x)\|,
 \qquad
 G_{\ell,N}(x)\mid\cF_{\ell-1}\sim N(0,\Gamma_{\ell,N}(x)).
\]
For the first layer, the same notation refers to the deterministic covariance
quantities of $b_{1,j}+\langle W_{1,j},x\rangle$. The second central proposition expresses the expected surface measure of one activation boundary through these conditional covariance quantities.

\begin{proposition}[Conditional Kac--Rice formula in fixed dimension]
\label{prop:higher-dimensional-kac-rice}
For every $1\le\ell\le L$ and $1\le j\le n_\ell$, we have
\(
 \E[\mathcal H^{d-1}(\mathcal Z_{\ell,j}(D))
 \mid\cF_{\ell-1}]
 =\int_D\kappa^{(d)}_{\ell,N}(x)\,\dd x,
\)
where the conditioning is omitted for $\ell=1$.
\end{proposition}

For fixed $x\in D$, define 
\(
 M^{(0)}_{\ell,N}(x)=\frac1{n_\ell}\sum_jh_{\ell,j}(x)^2,\) and
\[
 M^{(1)}_{\ell,N}(x)=\frac1{n_\ell}\sum_j
 h_{\ell,j}(x)\nabla h_{\ell,j}(x),\text{ and }
 M^{(2)}_{\ell,N}(x)=\frac1{n_\ell}\sum_j
 \nabla h_{\ell,j}(x)\nabla h_{\ell,j}(x)^{\mathsf T}.
\]
The second quantity is a column vector and the third is a symmetric matrix. The third central proposition propagates these empirical moments and identifies the deterministic covariance limit of the surface density.

\begin{proposition}[Propagation of the covariance matrices]
\label{prop:higher-dimensional-moments}
For every fixed $x\in D$ and every layer $\ell$, in probability and in $L^1$,
\[
 \left(M^{(0)}_{\ell,N}(x),M^{(1)}_{\ell,N}(x),
 M^{(2)}_{\ell,N}(x)\right)
 \longrightarrow
 \frac12\left(s_\ell(x),B_\ell x,B_\ell I_d\right).
\]
Consequently,
$\left(S_{\ell,N}(x),R_{\ell,N}(x),C_{\ell,N}(x)\right)\to
\left(s_\ell(x),B_\ell x,B_\ell I_d\right)$ in probability and in $L^1$.
\end{proposition}

The fourth central proposition identifies the codimension-one scalar kink set and bounds the surface measure lost through inactive downstream layers. It is the step that transfers the activation-boundary asymptotic to the scalar output.

\begin{proposition}[Visible facets and the scalar kink set]
\label{prop:higher-dimensional-visibility}
On the event in Proposition~\ref{lem:higher-dimensional-geometry}, the part of
codimension one in $\mathcal R_N(D)$ is the union of the regular switching
facets whose downstream sensitivity is nonzero. This union is disjoint up to a
set of zero $\mathcal H^{d-1}$ measure, and
\[
 \E\bigl[\mathsf S_{N,d}(D)-
 \mathcal H^{d-1}(\mathcal R_N(D))\bigr]
 \le
 \sum_{r=1}^L2^{-n_r}
 \E\mathsf S_{N,d}(D).
\]
	In particular, if $\E[\mathsf S_{N,d}(D)]=O(N)$ and
$N\sum_{r=1}^L2^{-n_r}\to0$, then the left side tends to zero.
\end{proposition}

\begin{proof}[Proof of Theorem~\ref{thm:higher-dimensional-boundaries}]
By Proposition~\ref{prop:higher-dimensional-kac-rice}, the expected measure of
one activation boundary is the integral of its conditional surface density.
Proposition~\ref{prop:higher-dimensional-moments} and the moment bounds proved
below show that, for every fixed layer and $x\in D$, the expectation of this
density converges to $\rho_{\ell,d}(x)$ and is bounded by a constant independent
of $N$ and $x$. Dominated convergence therefore gives
	$\E[\mathcal H^{d-1}(\mathcal Z_{\ell,1}(D))]=
\int_D\rho_{\ell,d}(x)\,\dd x+o(1)$. Exchangeability within each layer and the
fact that $L$ is fixed yield
\(
	\E[\mathsf S_{N,d}(D)]
 =\sum_{\ell=1}^L n_\ell\int_D\rho_{\ell,d}(x)\,\dd x+o(N).
\)
The integrals are finite because $D$ is bounded and $A_\ell>0$, so
$\E\mathsf S_{N,d}(D)=O(N)$.
Proposition~\ref{lem:higher-dimensional-geometry} shows that the intersections
of distinct switching hypersurfaces and the lower-dimensional skeleton have
zero $\mathcal H^{d-1}$ measure. Proposition
\ref{prop:higher-dimensional-visibility} therefore identifies the remaining
visible facets with the part of codimension one in the scalar kink set. Since
proportional widths imply $N\sum_r2^{-n_r}\to0$, the estimate in that
proposition gives
\(
 \E\bigl[\mathsf S_{N,d}(D)-
 \mathcal H^{d-1}(\mathcal R_N(D))\bigr]\to0.
\)
The hidden-boundary asymptotic is the auxiliary estimate. Substitution into it gives the scalar kink-set asymptotic stated in Theorem~\ref{thm:higher-dimensional-boundaries}. All sets are measured inside the open domain $D$, so no term from $\partial D$ appears.
\end{proof}

\subsection{Proof of auxiliary results}
We now prove the stated auxiliary results.
We first prove the genericity proposition. The deterministic masks index all
exceptional events before the network parameters are sampled. The same
representation will later identify the regular switching facets used in the
visibility argument.

\begin{proof}[Proof of Proposition~\ref{lem:higher-dimensional-geometry}]
Fix a mask $\varepsilon$ and a hidden neuron $u$. Once the gates are frozen, repeated substitution through the network shows that the masked preactivation is affine, with coefficients that are polynomials in the network parameters. Moreover, on $C_\varepsilon$ the frozen gates agree with the realized gates, so the masked and realized networks coincide there:
\[
 z_u^\varepsilon(x)=a_u^\varepsilon+\langle g_u^\varepsilon,x\rangle,
 \qquad
 h_u^\varepsilon(x)=\varepsilon_u z_u^\varepsilon(x),
 \qquad
 (z_v,h_v)|_{C_\varepsilon}=(z_v^\varepsilon,h_v^\varepsilon)|_{C_\varepsilon}
 \quad\text{for every }v.
\]
To exclude an identically zero masked preactivation, set $Q_{\varepsilon,u}=(a_u^\varepsilon)^2+\|g_u^\varepsilon\|^2$. The bias $b_u$ enters $a_u^\varepsilon$ with coefficient 1, so $Q_{\varepsilon,u}$ is a nonzero polynomial. Since the joint Gaussian parameter vector has a density, its zero set has probability 0. Hence, outside a finite union of null sets, every nonempty $F_{\varepsilon,u}$ is a relatively open subset of a genuine affine hyperplane with nonzero normal.

We next rule out coincident facets. Consider distinct neurons $u$ and $v$, first with $v$ not earlier than $u$. If two candidate hyperplanes agree on a relatively open set, their affine coefficient vectors are proportional. Choose a coordinate $r$ with $g_{u,r}^\varepsilon\ne0$ and define
\(
 P_{\varepsilon,\eta,u,v,r}
 =g_{u,r}^\varepsilon a_v^\eta-g_{v,r}^\eta a_u^\varepsilon.
\)
Coincidence of the two hyperplanes forces $P_{\varepsilon,\eta,u,v,r}=0$. However, $\partial P_{\varepsilon,\eta,u,v,r}/\partial b_v=g_{u,r}^\varepsilon\not\equiv0$, so this polynomial is nontrivial and therefore vanishes only on a null set. Thus, distinct activation boundaries cannot share a relatively open facet. Their intersection has dimension at most $d-2$, and hence $\mathcal H^{d-1}(\mathcal Z_u(D)\cap\mathcal Z_v(D))=0$. The same argument, with $u$ in an earlier layer and $v$ in a later layer, shows that a later preactivation cannot vanish on a relatively open part of an inherited facet. Taking the finite union over all lower-dimensional faces also gives $\mathcal H^{d-1}(\mathscr S_N)=0$.

It remains to exclude cancellation along an active downstream route. After the downstream gates are frozen, the sensitivity of $u$ is the polynomial
\(
 D_u^\varepsilon
 =\sum_{\pi:u\to{\rm out}}
 \Big(\prod_{e\in\pi}c_e w_e\Big)
 \Big(\prod_{v\in\pi\setminus\{u,{\rm out}\}}\varepsilon_v\Big).
\)
If $\varepsilon$ contains an active path from $u$ to the output, the monomial corresponding to that path has nonzero coefficient. No different path produces the same monomial, because two distinct paths differ in an edge variable at their first divergence. Thus, $D_u^\varepsilon$ is a nonzero polynomial, so $\Pp(D_u^\varepsilon=0)=0$. A finite union over masks and neurons proves the assertions.
\end{proof}

The conditional surface formula is the higher-dimensional analogue of the 1D zero-count identity. Its geometric input is the coarea formula; Federer gives the classical geometric-measure-theoretic treatment \cite{Federer1969}, while Hanin and Rolnick use coarea at finite width to study boundary volume in random ReLU networks \cite{HaninRolnick2019LinearRegions}. The additional point here is that the inherited polyhedral skeleton must not contribute $(d-1)$-dimensional measure to the zero set. We therefore isolate the exact piecewise-affine Gaussian identity needed for the network proposition and then verify it cell by cell.

Let $\mathcal P$ be a finite polyhedral complex in $\R^d$, let $D\subset\R^d$ be bounded and open, and let $f_1,\ldots,f_m$ be continuous functions that are affine on every full-dimensional cell of $\mathcal P$. Let $\zeta,\xi_1,\ldots,\xi_m$ be independent centered Gaussian variables with $\Var(\zeta)=\tau^2>0$, and define $Z(x)=\zeta+\sum_i\xi_i f_i(x)$. At points in the relative interiors of the full-dimensional cells, define $S(x)=\Var Z(x)$, $R(x)=\Cov(Z(x),\nabla Z(x))$, and $C(x)=\Cov(\nabla Z(x))$. The next lemma gives the exact surface identity needed for the conditional network calculation.

\begin{lemma}[Surface formula for a piecewise-affine Gaussian field]
\label{lem:piecewise-surface-kac-rice}
Under the preceding setup,
\[
	\E\big[\mathcal H^{d-1}(\{x\in D:Z(x)=0\})\big]
 =\int_D\frac{1}{\sqrt{2\pi S(x)}}
	\E\big[\|N(0,C(x)-R(x)R(x)^{\mathsf T}/S(x))\|\big]\,\dd x.
\]
The values of the integrand on the lower-dimensional skeleton may be chosen
arbitrarily.
\end{lemma}

\begin{proof}[Proof of Proposition~\ref{prop:higher-dimensional-kac-rice}]
Condition on $\cF_{\ell-1}$. The preceding activations are deterministic
functions that are affine on the cells of the finite complex from
Proposition~\ref{lem:higher-dimensional-geometry}, while
\(
 z_{\ell,j}(x)=b_{\ell,j}+\frac1{\sqrt{n_{\ell-1}}}
 \sum_iW_{\ell,ji}h_{\ell-1,i}(x).
\)
The bias $b_{\ell,j}$ has variance $\beta_\ell>0$, so for almost every
realization of the preceding layers the hypotheses of
Lemma~\ref{lem:piecewise-surface-kac-rice} hold. Applying that lemma to the conditional Gaussian law of the
current-layer parameters gives the asserted identity. In particular, the
codimension-one faces inherited from preceding layers contribute zero surface
measure by the final part of the lemma, so only the zero sets inside the
full-dimensional cells enter the integral. The first layer is the same
argument with one global affine cell.
\end{proof}

\begin{proof}
Fix a full-dimensional cell and write
$Z(x)=U+\langle V,x\rangle$ on its relative interior. The Gaussian law of $V$
may be singular. Since $\zeta$ occurs in $U$ and not in $V$, the conditional
law of $U$ given $V$ has variance at least $\tau^2$. For $v\ne0$, the coarea
formula applied to $x\mapsto U+\langle v,x\rangle$ gives, for every Borel set
$E$ in the cell,
\[
 \E\big[\mathcal H^{d-1}
 (\{U+\langle v,x\rangle=0\}\cap E)\mid V=v\big]
 =\int_E \|v\|p_{U\mid V=v}(-\langle v,x\rangle)\,\dd x.
\]
When $v=0$, the zero set is empty almost surely because the conditional law of
$U$ is nondegenerate, and both sides vanish. Integration with respect to $V$
and the Gaussian conditioning formula give
$p_{Z(x)}(0)\E[\|V\|\mid Z(x)=0]$, which is the stated integrand. This is the exact higher-dimensional analogue of the 1D identity above: the absolute slope $|V_k|$ is replaced by the gradient norm $\|V\|$, because the coarea formula converts that norm into $(d-1)$-dimensional surface measure.

It remains to check that summing over full-dimensional cells does not lose
surface measure on the skeleton. Faces of dimension at most $d-2$ have zero
$\mathcal H^{d-1}$ measure. Let $F$ be a face of dimension $d-1$. The
restriction of $Z$ to the affine hull of $F$ is affine. Unless this restriction
vanishes identically, its zero set in $F$ is contained in an affine set of
dimension at most $d-2$ and therefore has zero $\mathcal H^{d-1}$ measure.
Thus, a positive surface contribution from $F$ is possible only when $Z$
vanishes identically on $F$. Choosing any deterministic point $x_F$ in the
relative interior of $F$, this event is contained in $\{Z(x_F)=0\}$. Since
$\Var Z(x_F)\ge\tau^2$, it has probability zero. There are only finitely many
faces, so the entire skeleton contributes zero surface measure almost surely.
The relative interiors of the full-dimensional cells are disjoint, and the
openness of $D$ removes any contribution from $\partial D$.
\end{proof}

We next prove Proposition~\ref{prop:higher-dimensional-moments} and evaluate
the limiting surface density. At a point where an activation is not
differentiable, we set its gradient equal to zero. For every fixed
deterministic $x$, each preactivation has conditional variance at least its
positive bias variance, so the probability that $x$ lies on any switching
hypersurface is zero. The convention therefore does not affect identities at a
fixed input. For every fixed layer $\ell$, 
we have
\(
 \sup_N\sup_{x\in D}
 \E\bigl[|z_{\ell,1}(x)|^4+\|\nabla z_{\ell,1}(x)\|^4\bigr]<\infty.
\)
The assertion is immediate in the first layer. Conditional on
$\cF_{\ell-1}$, the preactivation is Gaussian with variance
$S_{\ell,N}$ and its gradient is a centered Gaussian vector with covariance
$C_{\ell,N}$. Hence,
$\E[z_{\ell,1}^4\mid\cF_{\ell-1}]=3S_{\ell,N}^2$ and
$\E[\|\nabla z_{\ell,1}\|^4\mid\cF_{\ell-1}]
\le3(\operatorname{tr}C_{\ell,N})^2$. The exact covariance formulas are
\(S_{\ell,N}(x)=\beta_\ell+\frac{\gamma_\ell}{n_{\ell-1}} \sum_i h_{\ell-1,i}(x)^2,\)
\(
 R_{\ell,N}(x)=\frac{\gamma_\ell}{n_{\ell-1}} \sum_i h_{\ell-1,i}(x)\nabla h_{\ell-1,i}(x),\) and
\(
 C_{\ell,N}(x)=\frac{\gamma_\ell}{n_{\ell-1}}
 \sum_i\nabla h_{\ell-1,i}(x) \nabla h_{\ell-1,i}(x)^{\mathsf T}.
\)
Jensen's inequality and the induction hypothesis bound the expectations of
$S_{\ell,N}^2$ and $(\operatorname{tr}C_{\ell,N})^2$ uniformly in $N$ and
$x$. Since $|h|\le|z|$ and $\|\nabla h\|\le\|\nabla z\|$, the same bounds
hold for the activations. Cauchy's inequality also gives the required uniform
bound for $\E[h^2\|\nabla h\|^2]$.

\begin{proof}[Proof of Proposition~\ref{prop:higher-dimensional-moments}]
For $\ell=1$, the pairs
$(z_{1,j}(x),\nabla z_{1,j}(x))=
(b_{1,j}+\langle W_{1,j},x\rangle,W_{1,j})$ are independent centered Gaussian
pairs with covariance data $(s_1(x),B_1x,B_1I_d)$. For every centered Gaussian
pair $(Z,G)$, central symmetry gives
\(
 \E[Z_+^2]=\tfrac12\E[Z^2],\)
\(\E[ZG\1\{Z>0\}]=\tfrac12\E[ZG],\) and
\(\E[GG^{\mathsf T}\1\{Z>0\}]=\tfrac12\E[GG^{\mathsf T}].\)
The ordinary law of large numbers and the fourth-moment bounds prove the first
layer assertion in $L^2$. The first-layer covariance vector is deterministic
and already equals $(s_1(x),B_1x,B_1I_d)$.

Suppose the conclusion holds through layer $\ell-1$. Set
$\mathbf M_{\ell,N}=(M^{(0)}_{\ell,N},M^{(1)}_{\ell,N},M^{(2)}_{\ell,N})$,
$\mathbf V_{\ell,N}=(S_{\ell,N},R_{\ell,N},C_{\ell,N})$, and
$\mathbf v_\ell=(s_\ell,B_\ell x,B_\ell I_d)$. Equip the product space with
the sum of the absolute value, the Euclidean norm, and the Frobenius norm. The
exact covariance identities give
\[
 \mathbf V_{\ell,N}
 =\bigl(\beta_\ell+\gamma_\ell M^{(0)}_{\ell-1,N},
        \gamma_\ell M^{(1)}_{\ell-1,N},
        \gamma_\ell M^{(2)}_{\ell-1,N}\bigr).
\]
The induction hypothesis and the recursions
$A_\ell=\beta_\ell+\gamma_\ell A_{\ell-1}/2$ and
$B_\ell=\gamma_\ell B_{\ell-1}/2$ therefore imply
$\mathbf V_{\ell,N}\to\mathbf v_\ell$ in $L^1$ and in probability.

It remains to control the empirical fluctuations within the new layer.
Conditional on $\cF_{\ell-1}$, the neurons in layer $\ell$ are independent and
identically distributed centered Gaussian pairs, so the symmetry identities
give $\E[\mathbf M_{\ell,N}\mid\cF_{\ell-1}]=\mathbf V_{\ell,N}/2$.
Conditional independence and the Gaussian fourth-moment bounds yield
\[
 \E\bigl[\|\mathbf M_{\ell,N}
 -\E[\mathbf M_{\ell,N}\mid\cF_{\ell-1}]\|^2
 \mid\cF_{\ell-1}\bigr]
 \le \frac{C_d}{n_\ell}
 \bigl(S_{\ell,N}^2+S_{\ell,N}\operatorname{tr}C_{\ell,N}
 +(\operatorname{tr}C_{\ell,N})^2\bigr).
\]
Indeed, the scalar component is bounded by
$\E[Z^4\mid\cF_{\ell-1}]=3S_{\ell,N}^2$, the vector component by
$\E[Z^2\|G\|^2\mid\cF_{\ell-1}]
\le3S_{\ell,N}\operatorname{tr}C_{\ell,N}$, and the matrix component by
$\E[\|G\|^4\mid\cF_{\ell-1}]
\le3(\operatorname{tr}C_{\ell,N})^2$. The moment bounds above and the
Cauchy--Schwarz inequality therefore give a constant $C_\ell$, independent of
$N$, such that
$\E\|\mathbf M_{\ell,N}-\E[\mathbf M_{\ell,N}\mid\cF_{\ell-1}]\|^2
\le C_\ell/n_\ell$. Consequently,
\[
 \E\|\mathbf M_{\ell,N}-\tfrac12\mathbf v_\ell\|
 \le \frac{C_\ell^{1/2}}{\sqrt{n_\ell}}
 +\frac12\E\|\mathbf V_{\ell,N}-\mathbf v_\ell\|\longrightarrow0.
\]
Hence, $\mathbf M_{\ell,N}\to\mathbf v_\ell/2$ in $L^1$ and therefore in
probability. This completes the induction.
\end{proof}
It remains to pass from the covariance data to the surface density. Recall that
\(
 \Gamma_{\ell,N}
 =C_{\ell,N}-\frac{R_{\ell,N}R_{\ell,N}^{\mathsf T}}{S_{\ell,N}},
 \) and
\( \Gamma_\ell
 =C_\ell-\frac{r_\ell r_\ell^{\mathsf T}}{s_\ell}.
\)
Since $S_{\ell,N}\ge\beta_\ell$ and $s_\ell\ge\beta_\ell$, we have
\[
\begin{aligned}
 \|\Gamma_{\ell,N}-\Gamma_\ell\|_F
 &\le \|C_{\ell,N}-C_\ell\|_F
 +\frac1{\beta_\ell}
   \|R_{\ell,N}R_{\ell,N}^{\mathsf T}
      -r_\ell r_\ell^{\mathsf T}\|_F
 +\frac{\|r_\ell r_\ell^{\mathsf T}\|_F}{\beta_\ell^2}
   |S_{\ell,N}-s_\ell|,
\end{aligned}
\]
while
\(
 \|R_{\ell,N}R_{\ell,N}^{\mathsf T}
      -r_\ell r_\ell^{\mathsf T}\|_F
 \le
 (\|R_{\ell,N}\|+\|r_\ell\|)
 \|R_{\ell,N}-r_\ell\|.
\)
Proposition~\ref{prop:higher-dimensional-moments} therefore gives
\(
 \Gamma_{\ell,N}(x)\xrightarrow{\Pp}\Gamma_\ell(x).
\)
Both matrices are positive semidefinite. Hence, continuity of the principal
square root and
\(
 \big|
 \E_G\|\Gamma_{\ell,N}^{1/2}G\|
 -\E_G\|\Gamma_\ell^{1/2}G\|
 \big|
 \le
 \|\Gamma_{\ell,N}^{1/2}-\Gamma_\ell^{1/2}\|_F
\)
give
\(
 \kappa^{(d)}_{\ell,N}(x)\xrightarrow{\Pp}\rho_{\ell,d}(x),
\)
where $\E_G$ denotes expectation only with respect to $G\sim N(0,I_d)$.
The middle inequality, together with continuity of the principal square root,
records the only continuity input. Jensen's inequality and
$\Gamma_{\ell,N}\le C_{\ell,N}$ give
\[
\begin{aligned}
 \E[(\kappa^{(d)}_{\ell,N}(x))^2]
 &\le\frac{\E\operatorname{tr}C_{\ell,N}(x)}{2\pi\beta_\ell},\qquad
 \sup_N\sup_{x\in D}\E[(\kappa^{(d)}_{\ell,N}(x))^2]&<\infty,
 \qquad
 \E\kappa^{(d)}_{\ell,N}(x)\to\rho_{\ell,d}(x).
\end{aligned}
\]
Thus, the surface densities are uniformly integrable at each fixed input, and
the same bound gives a deterministic integrable majorant on the bounded domain
$D$. Dominated convergence and
Proposition~\ref{prop:higher-dimensional-kac-rice} now yield
\[
 \E\mathcal H^{d-1}(\mathcal Z_{\ell,1}(D))
 =\int_D\E\kappa^{(d)}_{\ell,N}(x)\,\dd x
 \longrightarrow\int_D\rho_{\ell,d}(x)\,\dd x.
\]

Finally, we now prove Proposition~\ref{prop:higher-dimensional-visibility}. The geometric part identifies the codimension-one kink set with the visible switching facets. The probabilistic part then conditions on a switching facet and uses the probability that an entire downstream layer is inactive to bound the invisible surface measure. The resulting estimate is the higher-dimensional analogue of the visibility bound used for 1D switch counts.

\begin{proof}[Proof of Proposition~\ref{prop:higher-dimensional-visibility}]
Work on the event of Proposition~\ref{lem:higher-dimensional-geometry}, and
remove from every switching facet its intersections with other switching
hypersurfaces and all faces of dimension at most $d-2$. The removed set has
zero $\mathcal H^{d-1}$ measure. On the remaining regular part of the facet
of a neuron $u$, only the gate of $u$ changes. Freezing all downstream gates
therefore gives
	\(f_N^+=D_u h_u^++g\),  \(f_N^-=D_u h_u^-+g,\)
	\(\nabla f_N^+-\nabla f_N^-
	=D_u(\nabla h_u^+-\nabla h_u^-),\) and
\(
 \nabla h_u^+-\nabla h_u^-=\pm\nabla z_u\ne0.
\)
Hence, on regular switching facets,
$\nabla f_N^+=\nabla f_N^-$ if and only if $D_u=0$. Proposition~\ref{lem:higher-dimensional-geometry}
also shows that all remaining intersections form an $\mathcal H^{d-1}$-null
set and that distinct neurons do not share a relatively open facet. Thus,
the codimension-one part of $\mathcal R_N(D)$ is exactly the union of the
regular switching facets with $D_u\ne0$, disjoint up to an
$\mathcal H^{d-1}$-null set.

Fix $u=(\ell,j)$ with $\ell<L$, put
$\mu_u=\mathcal H^{d-1}\!\restriction\mathcal Z_u(D)$, and condition on
$\cF_\ell$. For $r>\ell$, let $A_r(x)$ be the event that every neuron in
layer $r$ is inactive at $x$. Conditional on $\cF_{r-1}$, the signs in layer
$r$ are independent and fair. If none of the events $A_r(x)$ occurs, full
connectivity supplies an active path from $u$ to the output, so
Proposition~\ref{lem:higher-dimensional-geometry} gives $D_u(x)\ne0$ outside
a $\mu_u$-null set. Therefore,
\(\Pp(A_r(x)\mid\cF_\ell)=2^{-n_r},\) and \(\{D_u=0\}\subseteq\bigcup_{r=\ell+1}^L A_r\) 
 \(\mu_u\text{-a.e.},\)
  and
\(
 \begin{aligned}
 \E[\mu_u(\{D_u=0\})\mid\cF_\ell]
 &\le\sum_{r=\ell+1}^L2^{-n_r}\,\mu_u(D).
 \end{aligned}
\)
For $\ell=L$, the output weight is nonzero almost surely, so every regular
last-layer facet is visible. Summing the preceding estimate over all neurons
and taking expectations yields
\[
 \begin{aligned}
  \E\bigl[\mathsf S_{N,d}(D)
       -\mathcal H^{d-1}(\mathcal R_N(D))\bigr]
 &\le\sum_{\ell=1}^{L-1}\sum_{j=1}^{n_\ell}
       \sum_{r=\ell+1}^L2^{-n_r}
       \E\mathcal H^{d-1}(\mathcal Z_{\ell,j}(D))
 &\le\sum_{r=1}^L2^{-n_r}
       \E\mathsf S_{N,d}(D),
 \end{aligned}
\]
which proves the proposition.
\end{proof}

\section{Explicit constants and numerics in one input dimension}
\label{sec:numerics}

The recursion defining $A_\ell$ and $B_\ell$ is affine in $\ell$ and can be solved in closed form when the variances do not depend on the layer. Fix constants $\beta>0$ and $\gamma>0$ and specialize to $\beta_\ell=\beta$ and $\gamma_\ell=\gamma$ for every $\ell$. The next corollary records the resulting closed forms and the corresponding integrated switch intensity.

\begin{corollary}[Intensities for constant variances]
\label{cor:constant-variance}
Under the constant-variance specialization above, for every compact interval
$I=[a,b]$, we have 
	\( \int_I\rho_\ell(t)\,\dd t
 =\frac1\pi[
 \arctan\bigl(b\sqrt{B_\ell/A_\ell}\bigr)
	-\arctan\bigl(a\sqrt{B_\ell/A_\ell}\bigr)]\) where
\[
 \begin{gathered}
 B_\ell=\gamma(\gamma/2)^{\ell-1},\qquad
 A_\ell=
 \begin{cases}
 \beta\dfrac{1-(\gamma/2)^\ell}{1-\gamma/2},&\gamma\ne2,\\[1.2ex]
 \beta\ell,&\gamma=2.
 \end{cases}
 \end{gathered}
\]
For $\beta=\gamma=1$ and $I=[-1,1]$, the integral is
$2\pi^{-1}\arctan((2^\ell-1)^{-1/2})$; it equals $1/2$ for $\ell=1$, $1/3$
for $\ell=2$, and $2\pi^{-1}\arctan(1/\sqrt7)\approx0.2301$ for $\ell=3$.
\end{corollary}

\begin{proof}
The recursion $B_\ell=(\gamma/2)B_{\ell-1}$ with $B_1=\gamma$ gives the geometric formula, and $A_\ell=\beta+(\gamma/2)A_{\ell-1}$ with $A_1=\beta$ is a geometric sum.  The integral formula is the arctangent antiderivative of the density of Cauchy form $\rho_\ell$.  For $\beta=\gamma=1$ the ratio simplifies to $B_\ell/A_\ell=1/(2^\ell-1)$, and $\arctan1=\pi/4$, $\arctan(1/\sqrt3)=\pi/6$ give the stated values.
\end{proof}

For $\gamma<2$ the slope variance $B_\ell$ decays geometrically while $A_\ell$ converges, so deep layers switch rarely on a fixed interval; $\gamma=2$ is the threshold scaling at which $B_\ell$ remains of constant order. In our normalization, $\gamma=2$ corresponds to the effective weight variance $2/n_{\ell-1}$ of the {ReLU}-adapted initialization of He, Zhang, Ren, and Sun \cite{HeZhangRenSun2015}, and the halving factor in the recursion is the same half-space identity $\E[\relu(Z)^2]=\tfrac12\Var(Z)$ that underlies their argument for preserving variance.

\begin{table}[b]
\caption{Monte Carlo estimates of $1+\E R_N(I)$, the expected number of maximal affine intervals, for $\beta=\gamma=1$, $I=[-1,1]$, equal widths $n_\ell=n$, based on $2000$, $1000$, and $600$ independent networks for $L=1,2,3$; the last block compares two unequal allocations of the same total width $N=192$, based on $500$ networks each. The theory column shows $1+\sum_\ell n_\ell\int_I\rho_\ell$: for $L=1$ this value is exact, while for $L\ge2$ it is the leading prediction of Theorem~\ref{thm:affine-region-asymptotics}.}
\label{tab:mc}
\begin{tabular}{ccccc}
\hline
$L$ & $n$ & $N$ & $1+\E R_N(I)$ (MC) & theory\\
\hline
$1$ & $16$ & $16$ & $8.94\pm0.05$ & $9$\\
$1$ & $64$ & $64$ & $33.08\pm0.09$ & $33$\\
$1$ & $256$ & $256$ & $129.15\pm0.17$ & $129$\\
\hline
$2$ & $16$ & $32$ & $14.16\pm0.11$ & $14.33$\\
$2$ & $64$ & $128$ & $53.92\pm0.21$ & $54.33$\\
$2$ & $256$ & $512$ & $213.79\pm0.43$ & $214.33$\\
\hline
$3$ & $16$ & $48$ & $17.48\pm0.18$ & $18.01$\\
$3$ & $64$ & $192$ & $68.72\pm0.37$ & $69.06$\\
$3$ & $128$ & $384$ & $136.02\pm0.50$ & $137.11$\\
\hline
$2$ & $(128,64)$ & $192$ & $86.17\pm0.36$ & $86.33$\\
$2$ & $(64,128)$ & $192$ & $74.71\pm0.42$ & $75.67$\\
\hline
\end{tabular}
\end{table}

\begin{remark}[The constants as depth increases]
\label{rem:depth-saturation}
For $\beta=\gamma=1$ and $I=[-1,1]$, the identity $\arctan\bigl((2^\ell-1)^{-1/2}\bigr)=\arcsin\bigl(2^{-\ell/2}\bigr)$ rewrites the intensity integrals of Corollary~\ref{cor:constant-variance} as $\int_I\rho_\ell=\tfrac2\pi\arcsin(2^{-\ell/2})$.  For equal widths $n_1=\cdots=n_L=n$, the leading term of Theorem~\ref{thm:affine-region-asymptotics} is therefore $c_Ln$ with
\(
        c_L=\frac2\pi\sum_{\ell=1}^{L}\arcsin\bigl(2^{-\ell/2}\bigr),
\)
and $c_L$ increases to $c_\infty=\tfrac2\pi\sum_{\ell=1}^{\infty}\arcsin(2^{-\ell/2})\approx1.6094$.  Thus, at a fixed common width, the first hidden layer contributes $n/2$ expected kinks to the leading count, while all deeper layers combined contribute at most about $1.11\,n$. The affine-region count is larger by one.  This is a statement about the deterministic constants only: Theorem~\ref{thm:affine-region-asymptotics} concerns fixed depth, and no limit theorem in which the depth grows is claimed.
\end{remark}

Two features of Theorem~\ref{thm:affine-region-asymptotics} can be tested directly.  First, for $L=1$ the leading term is exact at every finite width: the case $\ell=1$ of Proposition~\ref{prop:conditional-kac-rice} is an identity, every switch in the first layer is visible because its downstream sensitivity $a_j/\sqrt{n_1}$ is nonzero almost surely, and Lemma~\ref{lem:partition-identities} then gives
\(
        \E R_N(I)=n_1\int_I\rho_1(t)\,\dd t
        \;\text{ for every }n_1.
\)
With $\beta=\gamma=1$ and $I=[-1,1]$ this reads $\E R_N(I)=n_1/2$, so the expected number of maximal affine intervals is $1+n_1/2$. Second, for $L\ge2$ the theorem determines $\E R_N(I)$ only up to an uncontrolled $o(N)$ correction, and the quality of the affine-region prediction $1+\sum_\ell n_\ell\int_I\rho_\ell$ at finite width is an empirical question.
Table~\ref{tab:mc} and Figure~\ref{fig:mc-convergence-main} report a Monte Carlo experiment with $\beta=\gamma=1$, $I=[-1,1]$, and equal widths $n_1=\dots=n_L=n$.  All region counts are computed exactly by piecewise-affine propagation: switches in the first layer are the roots $-b_{1,j}/w_{1,j}$, switches in later layers are obtained by solving the affine restriction of each preactivation on each inherited cell, and $R_N(I)$ is the number of slope changes of $f_N$ across the ordered switch locations; hence $1+R_N(I)$ is the affine-region count.  No discretization of $I$ is involved. The computations for Table~\ref{tab:mc} and Figure~\ref{fig:mc-convergence-main} were carried out by the script \nolinkurl{verify_theorem.py}, with a fixed random seed recorded in the file.

The $L=1$ rows agree with the exact value within Monte Carlo error at every width.  For $L=2$ and $L=3$ the difference between the measured expectation and the leading prediction stays below two affine regions across all widths considered, although Theorem~\ref{thm:affine-region-asymptotics} only guarantees $o(N)$; sharpening this error bound is left open.  The normalized counts in Figure~\ref{fig:mc-convergence-main} are close to their limiting values over the displayed range.  The contribution per neuron decreases with depth, reflecting the geometric decay of $B_\ell/A_\ell$ in Corollary~\ref{cor:constant-variance}.

\begin{remark}[Widths that differ between layers]
\label{rem:unequal-widths}
Theorem~\ref{thm:affine-region-asymptotics} allows the widths to have different limiting proportions, and the equal widths in Table~\ref{tab:mc} and Figure~\ref{fig:mc-convergence-main} are chosen purely for presentation; for general widths the leading prediction weights the layer integrals by $n_\ell$.  The last block of Table~\ref{tab:mc} compares two allocations of the same budget $N=192$ at depth $L=2$ and shows that the prediction is equally accurate for unequal widths.  It also illustrates that transferring neurons toward earlier layers increases the expected count, in line with the strict decay of $\int_I\rho_\ell$ in $\ell$ from Corollary~\ref{cor:constant-variance}.
\end{remark}

\section{Numerical verification in higher input dimension}
\label{sec:higher-dimensional-numerics}

Section~\ref{sec:numerics} verified Theorem~\ref{thm:affine-region-asymptotics}
for a scalar input. We now verify
Theorem~\ref{thm:higher-dimensional-boundaries} directly, for input dimensions
$d=2,3,5$. The quantity to be estimated is a $(d-1)$-dimensional Hausdorff
measure. Locating the activation boundaries by evaluating the network on a
regular lattice covering $D$ is not practical beyond $d=2$, because the number
of lattice points needed to resolve a surface at a fixed accuracy grows
exponentially in $d$. We therefore estimate the measure through the Crofton
representation of Section~\ref{sec:higher-dimensional-boundaries}, which
expresses it by 1D zero counts of the kind already used in
Section~\ref{sec:exact-conditional-kac-rice}.

\subsection{Observation domain and Crofton estimator}
\label{subsec:hd-domain}

We choose a radially symmetric observation domain so that the theoretical surface integral can be evaluated accurately without introducing another high-dimensional numerical approximation. The same symmetry also makes the Crofton sampling distribution particularly simple. We first record the scaling of the theoretical integral and then describe the line-based estimator used in the experiments.
Throughout this section $D=B(0,R)$ is the open ball of radius $R$ centered at
the origin. Since $\rho_{\ell,d}(x)$ depends on $x$ only through $\|x\|$, the
integral $\int_D\rho_{\ell,d}\dd x$ reduces to a 1D radial
integral, which we evaluate by quadrature.

A bounded domain is not a matter of convenience here. As $\|x\|\to\infty$ one
has $s_\ell(x)\sim B_\ell\|x\|^2$ and $A_\ell/s_\ell(x)\to0$, so
\[
 \rho_{\ell,d}(x)\;\sim\;\frac{c_d}{\|x\|},
 \qquad
 c_d=\frac{\Gamma(d/2)}{\sqrt\pi\,\Gamma\bigl(\tfrac{d-1}{2}\bigr)} ,
\]
and therefore $\int_{\R^d}\rho_{\ell,d}\dd x=\infty$ for every $d\ge2$. For
$d=1$ the transverse term is absent, $\rho_\ell(t)$ decays like $t^{-2}$, and
the integral over the whole line converges. This is why the 1D
statement admits a global constant while
Theorem~\ref{thm:higher-dimensional-boundaries} is stated on a bounded convex
domain.

Fixing $R=1$ costs no generality. Substituting $x=Ru$ gives
\begin{equation}
\label{eq:domain-scaling}
 \int_{B(0,R)}\rho_{\ell,d}\bigl[A_\ell,B_\ell\bigr]\dd x
 =R^{\,d-1}\int_{B(0,1)}\rho_{\ell,d}\bigl[A_\ell,B_\ell R^2\bigr]\dd x ,
\end{equation}
where the bracket records the two covariance parameters entering
$s_\ell$. Changing the radius is thus the same as changing $B_\ell$, and the
tables below may be read at any radius through~\eqref{eq:domain-scaling}. We
take $R=1$ and $\beta_\ell=\gamma_\ell=1$ for all $\ell$.

We now turn to the estimator itself. Applying the Crofton formula in the normalization used above \cite{SchneiderWeil2008}, for a $(d-1)$-rectifiable set $S\subset\R^d$ we have
\[
 \mathcal H^{d-1}(S)
 =\frac{1}{2v_{d-1}}\int_{S^{d-1}}\int_{v^\perp}
 \#\bigl(S\cap(y+\R v)\bigr)\dd y\,\sigma(\dd v),
\]
with $\sigma$ the surface measure on $S^{d-1}$ and $v_{d-1}$ the volume of the
unit ball in $\R^{d-1}$. The orthogonal projection of $D=B(0,R)$ onto
$v^\perp$ is the $(d-1)$-dimensional ball of radius $R$. Sampling $v$
uniformly on $S^{d-1}$ and $y$ uniformly on that projection, the line
$\{y+tv:t\in\R\}$ meets $D$ in a segment, which we call a chord, and the
formula becomes
\begin{equation}
\label{eq:crofton-estimator}
 \mathcal H^{d-1}(S\cap D)
 =\frac{|S^{d-1}|\,R^{\,d-1}}{2}\,
 \E\bigl[\#\{t:\ |t|<T(y),\ y+tv\in S\}\bigr],
 \qquad T(y)=\sqrt{R^2-\|y\|^2}.
\end{equation}
The right-hand side is an expectation over the chord alone, so the empirical
average over independently sampled chords is unbiased and the only error is a
Monte Carlo error.

Along a chord the network is a piecewise-affine function of the single
variable $t$. The zeros of every preactivation are therefore obtained exactly
by the recursion used in
Section~\ref{sec:proof-affine-region-asymptotics}: the first layer
contributes the roots of an affine function, and on each cell inherited from
the preceding layers a new preactivation is again affine, so its zero is
found by solving a linear equation. No discretization of $D$ is involved, and
the measurement of $\mathsf S_{N,d}(D)$ uses only 1D counts.

\subsection{Calibration and network experiments}
\label{subsec:crofton-calibration}

Before applying the estimator to a network, we calibrate it on surfaces whose measure is known exactly. This checks both the Crofton normalization and the Monte Carlo implementation independently of the network calculation. We then apply the same estimator, without changing the sampling scheme, to the activation boundaries of finite networks.

For a sphere of radius $a<R$, the exact measure is
$\mathcal H^{d-1}=|S^{d-1}|a^{d-1}$. Table~\ref{tab:crofton-calibration}
reports the estimate from $4\times10^4$ chords.

\begin{table}[t]
\centering
\caption{The estimator \eqref{eq:crofton-estimator} applied to the sphere of
radius $a$ inside the unit ball. Errors are one standard error over chords.
The larger relative errors in the rows with $a=0.4$ reflect the small
proportion of chords that meet a small sphere.}
\label{tab:crofton-calibration}
\begin{tabular}{ccrrr}
\hline
$d$ & $a$ & exact & estimate & ratio\\
\hline
2 & 0.4 & 2.5133 & $2.4698\pm0.0153$ & 0.983\\
2 & 0.7 & 4.3982 & $4.4011\pm0.0144$ & 1.001\\
3 & 0.4 & 2.0106 & $2.0103\pm0.0230$ & 1.000\\
3 & 0.7 & 6.1575 & $6.1563\pm0.0314$ & 1.000\\
5 & 0.4 & 0.6738 & $0.7021\pm0.0212$ & 1.042\\
5 & 0.7 & 6.3192 & $6.3093\pm0.0562$ & 0.998\\
\hline
\end{tabular}
\end{table}

We next apply the calibrated estimator to the network boundaries. Table~\ref{tab:crofton-networks} compares the asymptotic prediction
$\sum_\ell n_\ell\int_D\rho_{\ell,d}\dd x$ of
Theorem~\ref{thm:higher-dimensional-boundaries} with the estimate obtained
from $40$ independent networks and $300$ chords per network.

\begin{table}[t]
\centering
\caption{Expected boundary measure $\E\mathsf S_{N,d}(D)$ on the unit ball,
with $\beta_\ell=\gamma_\ell=1$. Errors are one standard error over the $40$
networks.}
\label{tab:crofton-networks}
\begin{tabular}{ccrrrr}
\hline
$d$ & widths & $N$ & prediction & estimate & ratio\\
\hline
2 & $64$          & 64 & 76.68  & $76.22\pm1.18$  & 0.994\\
2 & $32,32$       & 64 & 64.34  & $61.23\pm1.16$  & 0.952\\
2 & $24,24,24$    & 72 & 61.81  & $59.87\pm1.39$  & 0.969\\
3 & $64$          & 64 & 128.00 & $128.74\pm1.53$ & 1.006\\
3 & $32,32$       & 64 & 107.83 & $108.30\pm1.84$ & 1.004\\
3 & $24,24,24$    & 72 & 103.80 & $104.42\pm2.44$ & 1.006\\
5 & $64$          & 64 & 210.55 & $211.04\pm2.12$ & 1.002\\
5 & $32,32$       & 64 & 178.13 & $178.63\pm2.78$ & 1.003\\
5 & $24,24,24$    & 72 & 171.86 & $167.43\pm3.35$ & 0.974\\
\hline
\end{tabular}
\end{table}

The ratios lie close to one at all three input dimensions and all three
depths. The deviations below one occur at the smaller input dimension and the
larger depths, which is the regime in which the $o(N)$ term of
Theorem~\ref{thm:higher-dimensional-boundaries} is least negligible at the
widths used here; the sign agrees with the corresponding gaps in
Table~\ref{tab:mc}. We do not read a rate from these numbers, and the theorem
makes no assertion about one.

All entries of Tables~\ref{tab:crofton-calibration}
and~\ref{tab:crofton-networks} were produced by a single script,
\nolinkurl{crofton_estimator.py}. It is separate from the script used for
Section~\ref{sec:numerics}, which assumes a scalar input; here the network is
restricted to a chord in $\R^d$, so the first layer is treated differently
while the propagation through the later layers is the same. Both scripts fix
their random seeds and record them.

\section{Scope of the theoretical results}
\label{sec:scope-extensions}

In the 1D model, the finite Gaussian geometry and the visibility estimate identify the leading expected number of activation switches with the expected number of scalar kinks; the number of affine regions is one larger. All limits in the widths assume fixed depth. The covariance propagation is pointwise in the path parameter. Together with uniform integrability, this suffices for the expectation asymptotics, but we do not prove a uniform process limit, variance bounds, or concentration.

Positive hidden bias variances prevent degeneracy on inherited cells. The visibility estimates use centered Gaussian preactivations, for which every sign in a new layer is fair. A noncentered model would require a uniform positive lower bound on the activation probabilities.

For fully connected networks in fixed input dimension, the proof of Theorem~\ref{thm:higher-dimensional-boundaries} first obtains an explicit asymptotic formula for the expected hidden activation-boundary measure and then transfers it to the scalar kink set. The input dimension and the observation domain are fixed, and the theorem concerns only the part of codimension one. It does not determine the number of full-dimensional affine regions. That problem would also require control of intersections in higher codimension and of the topology of the complement. The theorem does not provide variance bounds, concentration, curvature measures, or limits in which the input dimension grows.

The numerical sections illustrate the theorems but do not extend them. The Crofton estimator of Section~\ref{sec:higher-dimensional-numerics} is unbiased, so the entries of Tables~\ref{tab:crofton-calibration} and ~\ref{tab:crofton-networks} carry only a Monte Carlo error. The agreement observed there concerns the leading term at the widths used and does not establish a rate for the $o(N)$ correction. The observation domain is a ball; for another bounded convex domain in the layer integrals would have to be evaluated again, although the dependence on the radius is explicit through~\eqref{eq:domain-scaling}.

Convolutional architectures are outside the present scope. Xiong et al. study linear-region counts for convolutional neural networks \cite{XiongEtAl2020}, but the finite-network visibility argument used here would require additional work in that setting. Shared filters make preactivations at different spatial positions dependent within an output channel and can create coincident switches.

\section{Wisconsin experiment: complete exploratory report}
\label{app:wisconsin}
This appendix records the full exploratory design and numerical results used to form Section~\ref{sec:wisconsin}. All trained-network conclusions are empirical.

\subsection{Data, network, and segment families}
This subsection records the common setup for the exploratory Wisconsin experiments. The design keeps the theoretical initialization model unchanged while measuring the same switch statistic along several families of held-out segments before and after training. The four architectures are used to separate the main example from checks of width and depth. The Breast Cancer Wisconsin diagnostic data have $569$ observations, with $212$ malignant and $357$ benign cases, and $30$ features. In the exploratory run, features are standardized coordinatewise over the full dataset and a fixed $70/30$ train/test split with seed $42$ is used. All real-data segments join held-out test points. The architectures are $(64,64)$, $(128,128)$, $(64,64,64)$, and $(128,128,128)$. The figures in this appendix come from the single exploratory split described
here; the replicated $10\times5$ design of Section~\ref{sec:wisconsin} covers
$(64,64)$ and $(64,64,64)$ only. Hidden biases have variance $\sigma_b^2=0.1$ in the main runs; He scaling is used for weights. Training uses full-batch Adam with learning rate $3\times10^{-3}$ for $300$ epochs, and reported test accuracies are $0.95$--$0.97$.

Four segment families are used: within malignant, within benign, between class, and off manifold. The control endpoints are drawn independently from $N(0,I_{30})$. Along every segment, all switches are located exactly by solving affine equations layer by layer; no grid in the path parameter is used. Scalar output kinks are counted separately across the ordered switch locations.

\subsection{Initialization agreement}
The first check asks whether the closed-form leading initialization prediction remains accurate on the actual data geometry in dimension $30$. For each segment, we evaluate the leading term in equation~\eqref{eq:chord-prediction} from its endpoints without simulation or quadrature and sum the layerwise contributions. Across every architecture and segment family, measured initialization counts agree with this prediction within one to two standard errors. Representative entries are given in Table~\ref{tab:wis-init}.
\tblcap{Initialization agreement in the exploratory study: the leading
prediction of~\eqref{eq:chord-prediction} against the measured switch count, for
representative combinations of architecture and segment family. Errors are one
standard error over the network seeds.}\label{tab:wis-init}
\begin{center}
\begin{tabular}{llrr}
\toprule
architecture & family & prediction & measured\\
\midrule
$(64,64)$ & within malignant & 44.64 & $44.12\pm2.55$\\
$(64,64)$ & between class & 70.95 & $70.41\pm1.73$\\
$(128,128)$ & within benign & 91.80 & $90.69\pm3.07$\\
$(128,128)$ & control & 121.90 & $121.64\pm1.87$\\
$(64,64,64)$ & between class & 104.19 & $103.01\pm3.96$\\
$(128,128,128)$ & within malignant & 131.30 & $128.83\pm6.08$\\
\bottomrule
\end{tabular}
\end{center}

\subsection{Training-induced changes}
The second part of the experiment asks how the same switch statistic changes under supervised training. We separate within-class, between-class, and off-manifold segments so that a global change in network geometry can be distinguished from a change tied to the class relation of the endpoints. Table ~\ref{tab:wis-change} reports the raw relative change of the total switch count from initialization to the trained network.
\tblcap{Relative change in total switch count from initialization to epoch
$300$, for the four architectures of the single exploratory split.
Table~\ref{tab:wis-main} reports the two replicated architectures over the
$10\times5$ split--seed design.}\label{tab:wis-change}
\begin{center}
\begin{tabular}{lrrrr}
\toprule
architecture & within mal. & within ben. & between & control\\
\midrule
$(64,64)$ & $-23.0\%$ & $-22.3\%$ & $+18.4\%$ & $-4.8\%$\\
$(128,128)$ & $-19.2\%$ & $-20.4\%$ & $+13.5\%$ & $-6.2\%$\\
$(64,64,64)$ & $-28.1\%$ & $-27.4\%$ & $+15.0\%$ & $-9.4\%$\\
$(128,128,128)$ & $-25.0\%$ & $-24.9\%$ & $+5.9\%$ & $-14.9\%$\\
\bottomrule
\end{tabular}
\end{center}
Relative to the off-manifold control, the within-class contrast is about $-10$ to $-19$ percentage points and the between-class contrast about $+20$ to $+24$ percentage points across the four architectures of the exploratory split. We interpret this subtraction only as a comparison relative to the control.

For $(64,64,64)$, the layer decomposition is given in Table~\ref{tab:wis-layer-app}.
\tblcap{Layerwise relative change for the $(64,64,64)$ network on the
exploratory split. Table~\ref{tab:wis-layer-main} reports the same quantities
over the replicated $10\times5$ split--seed design.}\label{tab:wis-layer-app}
\begin{center}
\begin{tabular}{lrrr}
\toprule
family & layer 1 & layer 2 & layer 3\\
\midrule
within malignant & $-8.2\%$ & $-28.6\%$ & $-49.0\%$\\
between class & $+8.1\%$ & $+9.6\%$ & $+28.5\%$\\
off manifold & $+0.4\%$ & $-10.8\%$ & $-18.6\%$\\
\bottomrule
\end{tabular}
\end{center}
The first control layer is essentially flat in every architecture, between $+0.1\%$ and $+0.4\%$.

\subsection{Training trajectory and robustness}
The endpoint comparison does not show when the separation develops or whether it depends strongly on the training duration. We therefore follow one continuous optimization trajectory and separately vary the hidden bias variance and the number of epochs. For $(64,64)$, with $20$ network seeds and $200$ segments per family, the mean total switch counts at selected epochs are given in Table~\ref{tab:wis-traj}.
\tblcap{Mean total switch count along the training trajectory of the $(64,64)$
network on the exploratory split, averaged over $20$ network seeds and $200$
segments per family. Figure~\ref{fig:wis-splits}(b) shows all twelve
checkpoints.}\label{tab:wis-traj}
\begin{center}
\begin{tabular}{rrrrr}
\toprule
epoch & within mal. & within ben. & between & control\\
\midrule
0 & 45.69 & 45.74 & 70.70 & 60.52\\
20 & 38.69 & 40.44 & 79.66 & 59.29\\
100 & 34.86 & 36.80 & 84.18 & 58.38\\
300 & 33.98 & 35.77 & 85.70 & 58.51\\
600 & 33.71 & 35.42 & 86.06 & 58.53\\
\bottomrule
\end{tabular}
\end{center}
The families separate monotonically and are close to saturation around epochs
$150$--$300$. This trajectory is a separate run: it uses its own seed stream and
its own sample of evaluation segments on the single exploratory split, so its
levels are not directly comparable with the split-averaged figures of
Table~\ref{tab:wis-main}. Within the trajectory itself the changes from epoch
$0$ to epoch $300$ are $-25.6\%$ on within-malignant, $-21.8\%$ on
within-benign, $+21.2\%$ between classes and $-3.3\%$ on the control, so the
ordering matches the replicated figures. The Adam state is carried continuously
through the checkpoints; an earlier implementation that reset it at each
checkpoint is described below.

For $(64,64)$, the sign pattern is also unchanged across all nine combinations of $\sigma_b^2\in\{0.01,0.1,1\}$ and training lengths $\{100,300,1000\}$. Magnitudes vary: at $\sigma_b^2=1$, within-class changes are about $-4$ to $-11\%$, between-class changes about $+34$ to $+41\%$, and the control about $+1\%$. In every table and run, the number of output kinks equals the total switch count exactly, both before and after training.

\subsection{Reproducibility notes}
The exploratory tables of this appendix and Figure~\ref{fig:wis-splits}(b) were
produced by \nolinkurl{wisconsin_geometry.py}, with base seed $42$. Independent
networks use \texttt{default\_rng([42,\,s])} for the tables,
\texttt{default\_rng([42,\,1000+s])} for the trajectory, and
\texttt{default\_rng([42,\,7000+s])} for the robustness sweep. The replicated
$10\times5$ design of Section~\ref{sec:wisconsin} and
Figure~\ref{fig:wis-splits} were produced by \nolinkurl{wisconsin_replication.py},
which imports the network, training and switch-counting routines from
\nolinkurl{wisconsin_geometry.py} so that both studies measure the same quantity
with the same code. Its streams are disjoint from those above: the split
partitions come from \texttt{StratifiedShuffleSplit} with
\texttt{random\_state}$=42$, the evaluation segments of split $i$ from
\texttt{default\_rng([42,\,5000+i])}, and network seed $k$ of split $i$ from
\texttt{default\_rng([42,\,i,\,k])}. The four scripts and their saved outputs are archived with a permanent
identifier \citep{OzkanHirsch2026Code}.

The replication script compares the scalar output kink count with the total
hidden switch count on every segment, before and after training, and records the
number of exact agreements in its saved output. Kinks are identified by
comparing the one-sided slopes with relative tolerance $10^{-8}$, which
separates a genuine slope change from floating-point noise. A switch whose
downstream sensitivity is very small can produce a slope change below this
tolerance and is then not counted as a kink; this accounts for the six
exceptions reported in Section~\ref{sec:wisconsin}.

The exploratory study included internal consistency checks because the
theoretical prediction and the trained-network measurements use separate code
paths. Two implementation errors were identified during these checks and
corrected before the results above were recorded. Neither correction changes the
design of the experiment, but both are relevant for reproducibility. First, an
early trajectory implementation reset the Adam state at each checkpoint; the
corrected run carries optimizer state continuously. It remains a separate run
with its own seed stream and its own segment sample, and is not a re-derivation
of the main-text figures. Second, an initialization routine initially failed to
pass $\sigma_b^2$ into the covariance recursion in the robustness calculation;
this was corrected before the results above were recorded.

\section*{CRediT authorship contribution statement}
\noindent \textbf{Recep {\"O}zkan:} Conceptualization, Methodology, Formal analysis, Software, Validation, Visualization, Investigation, Writing -- original draft, Writing -- review and editing, Funding acquisition.

\noindent \textbf{Christian Hirsch:} Conceptualization, Methodology, Formal analysis, Validation, Data Curation, Investigation, Writing -- original draft, Supervision, Writing -- review and editing.

\section*{Declarations}
\noindent {\bf Declaration of generative AI and AI-assisted technologies in the writing process.} During the preparation of this work, the authors used ChatGPT and Claude to assist with mathematical exploration and drafting, language editing, manuscript structure, and consistency checks. The authors independently verified all mathematical claims, derivations, and proofs and reviewed and edited all AI-assisted text. The authors take full responsibility for the content of the article.

\medskip
\noindent {\bf Funding.} This work was supported by the Scientific and Technological Research Council of T\"urkiye (T\"UB\.ITAK) under the T\"UB\.ITAK-B\.IDEB 2219 International Postdoctoral Research Fellowship Programme [grant number 1059B192501091].

\medskip
\noindent {\bf Declaration of competing interest.} The authors declare that they have no known competing financial interests or personal relationships that could have appeared to influence the work reported in this paper.

\medskip
\noindent {\bf Acknowledgements.} This article is based upon work from COST Action 24122 mSPACE, supported by COST (European Cooperation in Science and Technology), \href{https://www.cost.eu}{www.cost.eu}.

\medskip
\noindent {\bf Data availability.} The scripts \nolinkurl{verify_theorem.py}, \nolinkurl{crofton_estimator.py}, \nolinkurl{wisconsin_geometry.py} and \nolinkurl{wisconsin_replication.py}, which produce all tables and figures of this article, together with the random seeds used, are available at \url{https://github.com/Recep-Ozkan/conditional-kac-rice-relu} and archived at \url{https://doi.org/10.5281/zenodo.22881941}. The Breast Cancer Wisconsin (diagnostic) data are publicly available and are loaded directly from \texttt{scikit-learn}.

\bibliography{references}

\end{document}